\documentclass[twoside,11pt]{article}

\usepackage{graphicx} %
\usepackage[utf8]{inputenc} %
\usepackage[T1]{fontenc}    %
\usepackage{url}            %
\usepackage{booktabs}       %
\usepackage{amsfonts}       %
\usepackage{nicefrac}       %
\usepackage{amsmath}
\usepackage{bbold}
\usepackage{subfig}
\usepackage{amssymb}

\usepackage{jmlr2e}

\usepackage{xcolor} %

\usepackage{appendix}

\newcommand{\op}{\left(}
\newcommand{\cp}{\right)}
\newcommand{\ob}{\left[}
\newcommand{\cb}{\right]}
\newcommand{\R}{\mathbb{R}}
\newcommand{\E}{\mathbb{E}}
\renewcommand{\H}{\mathcal{H}}
\newcommand{\1}{\mathbb{1}}

\newcommand{\ep}{\varepsilon}

\newcommand{\C}{\mathcal{C}}

\newcommand{\N}{\mathcal{N}}

\newcommand{\td}{\widetilde}
\renewcommand{\Re}{\mathfrak{Re}}
\renewcommand{\L}{\mathcal{L}}

\DeclareMathOperator{\tr}{tr}

\DeclareMathOperator{\diag}{diag}

\DeclareMathOperator{\MMD}{MMD}

\usepackage{lastpage}
\jmlrheading{26}{2025}{1-\pageref{LastPage}}{3/25; Revised
9/25}{11/25}{25-0544}{Eliot Beyler and Francis Bach}
\ShortHeadings{title}{all authors LAST names}
\ShortHeadings{Optimal Denoising in Score-Based Generative Models}{Beyler and Bach}

\firstpageno{1}

\begin{document}

\title{Optimal Denoising in Score-Based Generative Models:\\ The Role of Data Regularity}

\author{\name Eliot Beyler \email eliot.beyler@inria.fr \\
       \addr INRIA, Ecole Normale Supérieure\\
       PSL Research University\\
       Paris, France
       \AND
       \name Francis Bach \email francis.bach@inria.fr \\
       \addr INRIA, Ecole Normale Supérieure\\
       PSL Research University\\
       Paris, France}

\editor{Jianfeng Lu}

\maketitle

\begin{abstract}%
Score-based generative models achieve state-of-the-art sampling performance by denoising a distribution perturbed by Gaussian noise.
In this paper, we focus on a single deterministic denoising step, and compare the optimal denoiser for the quadratic loss, we name ``full-denoising'', to the alternative ``half-denoising'' introduced by \citet{hyvarinenNoisecorrectedLangevinAlgorithm2025}.
We show that looking at the performance in terms of distance between distributions tells a more nuanced story, with different assumptions on the data leading to very different conclusions.
We prove that half-denoising is better than full-denoising for regular enough densities, while full-denoising is better for singular densities such as mixtures of Dirac measures or densities supported on a low-dimensional subspace.
In the latter case, we prove that full-denoising can alleviate the curse of dimensionality under a \emph{linear manifold hypothesis}.
\end{abstract}%


\section{Introduction}

Score-based generative models, or diffusion models \citep{sohl-dicksteinDeepUnsupervisedLearning2015,saremiNeuralEmpiricalBayes2019,hoDenoisingDiffusionProbabilistic2020,songScoreBasedGenerativeModeling2021}, achieve state-of-the-art sampling performance by denoising a distribution perturbed by Gaussian noise. This denoising is made in several steps by removing each time a fraction of the noise, which can be seen as the discretization of a stochastic differential equation.

Here we focus on a single denoising step. This setting enables a more in-depth study, that could be generalized to multiple steps in future works, but is also relevant in practice for several reasons.
Firstly, it has been used in recent work by \citet{saremiChainLogConcaveMarkov2023}, who propose an alternative formulation of diffusion models in which each step corresponds to log-concave sampling.
The method first reduces the noise level by averaging multiple measurements, then tries to approximate the data distribution with a one-step denoising.
Moreover, in the framework of \emph{stochastic localization} \citep[][]{montanariSamplingDiffusionsStochastic2023}, one simulates a process that localizes to the distribution of interest, but as this process is simulated in finite time, the author specifies that an additional step is needed, for example by taking the conditional expectation, hence denoising.
This final denoising step is also present in denoising diffusion models even if not stated explicitly.
Indeed, when one lacks regularity assumptions on the data distribution, proofs of convergence for diffusion models fail, and to overcome this, \citet{chenSamplingEasyLearning2023} use early stopping of the denoising process, as keeping a little bit of noise offsets the missing regularity.  
But then should we keep this extra noise or do an additional denoising step? This last step is examined here, and relies crucially on assumptions about the data distribution. 

Formally, given a random variable $X\in \R^d$ with distribution $\mu_X$, we will define $Y = X + \ep$ with $X$ and $\ep$ independent and $\ep \sim \N(0, \sigma^2 I)$. The ``optimal'' denoiser, in the sense of the measurable function that minimizes $\E [ \Vert X - f(Y)\Vert^2]$, is the conditional expectation $\varphi(y) = \E[X|Y=y]$.\footnote{For this conditional expectation to exist, we will always assume that $X$ is integrable, i.e., $\E[\Vert X\Vert] <\infty$.}
This value is related to the score, i.e., the gradient of the logarithm of the density $p_Y$ of $Y$, through Tweedie's formula \citep{robbinsEmpiricalBayesApproach1956,miyasawaEmpiricalBayesEstimator1961,efronTweediesFormulaSelection2011}:
\begin{equation}
\label{eq:tweedie}
\varphi(y) = \E[X|Y=y] = y + \sigma^2\nabla\log p_Y(y).
\end{equation}

But does it minimize the distance between the distribution $\mu_{X}$ and $\mu_{f(Y)}$ (in Wasserstein or any other distance between probability distributions)?

Of course, we want to limit the space of functions $f$ to ones we can compute.\footnote{In fact, any distance between $\mu_{X}$ and $\mu_{f(Y)}$ can be made zero by taking $f$ a transport map, but it will not be computable in general.} Knowing that in practice, $\nabla\log p_Y$ can be estimated by \emph{denoising score matching} \citep[][]{vincentConnectionScoreMatching2011}, it is reasonable to look for an expression involving $\nabla\log p_Y$. In a recent paper, \citet{hyvarinenNoisecorrectedLangevinAlgorithm2025} introduces half-denoising 
$$
\varphi_{1/2}(y) = \frac{y + \E[X|Y=y]}{2} = y + \frac{\sigma^2}{2}\nabla\log p_Y(y),
$$
and uses it to generate samples from $X$ with a modified Langevin algorithm.
Our aim in this paper is to study in more details the performance of this half-denoising step and to compare it to full-denoising.

The general philosophy here is that the hypotheses made on the data distribution are essential to assess the performance of the different denoising processes.
We distinguish between two kinds of assumptions made in the literature.
The first kind is to assume that the distribution of $X$ is regular enough, i.e., that it admits a density $p_X$ with respect to the Lebesgue measure, and that this density is smooth, for example that $x\mapsto \nabla\log p_X(x)$ is $L$-Lipschitz \citep[see, e.g.,][]{chenSamplingEasyLearning2023}.  
The second kind, incompatible with the first, is to assume that the distribution has some kind of singularities, i.e., that it is concentrated on low dimensional manifolds, with pockets of mass separated by areas of low densities.
This framework is known as the \emph{manifold hypothesis} \citep[see, e.g.,][]{tenenbaumGlobalGeometricFramework2000,bengioRepresentationLearningReview2013,feffermanTestingManifoldHypothesis2016}.

\medskip
\textbf{Contributions.} In this work, we make the following contributions: 
\begin{itemize}
    \item We show that half-denoising is better for regular enough densities,  in the sense that the distance  (in MMD -- maximum mean discrepancy -- and in Wasserstein-$2$ distance) between the original distribution and the denoised distribution is of order $O(\sigma^4)$, compared to $O(\sigma^2)$ for full-denoising. We thus formalize and extend the scaling in $O(\sigma^4)$ obtained by \citet{hyvarinenNoisecorrectedLangevinAlgorithm2025}. On the contrary, full-denoising is better for singular distributions such as Dirac measures or Gaussian with small variance compared to the additional noise (section \ref{sct:dist}).
    \item When the variable is supported on a lower-dimensional subspace, we show that there is a trade-off between full-denoising which reduces the Wasserstein distance by ensuring that the output belongs to the subspace, and half-denoising that reduces the Wasserstein distance on the lower-dimensional subspace (assuming a regular enough density). Moreover, in the case where the subspace is of small enough dimension compared to the full space, we show that full-denoising is adaptive to this low dimensional structure and thus alleviates the curse of dimensionality as the Wasserstein distance only depends on the distance between distributions on the lower-dimensional subspace (section \ref{sct:subspace}).
    \item We show that the denoising performance for a mixture of distributions with disjoint compact supports behaves as if we were denoising each variable independently, plus an exponentially decreasing term (section \ref{sct:mixt}).
    \item Finally, combining these results, we show that for a linear version of the manifold hypothesis, where the data distribution is supported on disjoint compact sets, each of these belonging to a (different) linear subspace of low dimension, full-denoising can alleviate the curse of dimensionality even if the support of the distribution itself spans the whole space as it adapts to the local linear structure of the distribution.
\end{itemize}

\section{Notations}
We introduce the following notations:
\begin{itemize}
    \item  For $\alpha \in \R$, we denote $\varphi_\alpha(y) = y +\alpha \sigma^2 \nabla \log p_Y(y)$, such that $\alpha = 1/2$ corresponds to half-denoising, $\alpha = 1$ to full-denoising (and $\alpha = 0$ to no denoising at all).
    \item $\Vert \cdot\Vert$ the euclidean norm on $\R^d$, $B(x,r)$ the Euclidean ball of center $x$ and radius $r$.
    \item $\N(\mu,\Sigma)$ the multivariate Gaussian distribution of mean $\mu$ and covariance matrix $\Sigma$.
    \item For a random variable $Z\in \R^n$, we write $\L(Z)$ its law, and, when it exists, $p_Z$ its density with respect to the Lebesgue measure. We write $Z \sim \mu$ if $\L(Z) = \mu$, $Z_1 \sim Z_2$ if  $\L(Z_1) = \L(Z_2)$ and $Z_1\bot Z_2$ if $Z_1$ and $Z_2$ are independent. We also denote, for $\xi \in \R^n$, $\hat{p}_Z(\xi)= \E[e^{i \xi\cdot Z}]$ the characteristic function of $Z$.
    \item For $p\geq 1$, we define the $p$-Wasserstein distance between $\L(Z_1)$ and $\L(Z_2)$ as
    $$
    W_p(\L(Z_1),\L(Z_2)) = \op\inf_{\Gamma: \Gamma_{z_1} = \L(Z_1), \Gamma_{z_2} = \L(Z_2)}\int \Vert z_1 - z_2\Vert^p d\Gamma(z_1,z_2)\cp^{1/p},
    $$
    where $\{\Gamma: \Gamma_{z_1} = \L(Z_1), \Gamma_{z_2} = \L(Z_2)\}$ is the set of distributions on $\R^d\times\R^d$ with marginals $\L(Z_1)$ and $\L(Z_2)$ \citep[see, e.g.,][]{peyreComputationalOptimalTransport2019}.
    \item For a reproducing kernel Hilbert space $\H$, with kernel $k : \R^d\times\R^d \mapsto \R$, we define the maximum mean discrepancy (MMD) \citep{grettonKernelTwosampleTest2012} between $\L(Z_1)$ and $\L(Z_2)$ as
    $$
    \MMD_k(\L(Z_1),\L(Z_2)) = \sup_{f\in \H, \Vert f\Vert_\H \leq 1} \op\E[f(Z_1)] - \E[f(Z_2)]\cp.
    $$
    \item For $k \in \{0,1,..\}\cup\{\infty\}$, $d_1,d_2$ integers, we write $\C^k(\R^{d_1},\R^{d_2})$ the functions from $\R^{d_1}$ to $\R^{d_2}$  with $k$ continuous derivatives. If $d_2 =1$, we simply write $\C^k(\R^{d_1})$.
    \item For a linear operator, $A: \R^{d_1}\rightarrow\R^{d_2}$, we write $\Vert A\Vert_\text{op}$ the operator norm of $A$, defined by $\Vert A \Vert_\text{op} = \sup_{x\neq 0}\frac{\Vert Ax\Vert}{\Vert x \Vert}$.
    \item $\nabla$ the gradient operator, $\nabla^2$ the Hessian operator, $\nabla\cdot$ the divergence operator and $\Delta$ the Laplacian, that will always be taken with respect to the space variable $x\in \R^d$.
\end{itemize}

\section{Half-Denoising is Better for Regular Densities}
\label{sct:dist}
In this section, we show that half-denoising is better for regular enough densities. We start by studying Gaussian variables, for which we have closed-form expressions of the various Wasserstein distances (section \ref{sct:gaussian}). It provides insights into the behavior of the denoiser $\varphi_\alpha$, but will also show that the bound we prove in the following section are tight in the dependence with respect to $\sigma$. 
Then in section \ref{sct:MMD}, inspired by \citet{hyvarinenNoisecorrectedLangevinAlgorithm2025}, we prove a bound on the distance between the characteristic functions, $\Vert \hat{p}_X(\xi) - \hat{p}_{\varphi_\alpha(Y)}(\xi)\Vert$, that translates to bounds in MMD between the initial and the denoised distribution, under regularity assumptions on $p_X$. 
Finally in section \ref{sct:wasserstein}, we prove similar bounds in Wasserstein distance, by making a link between half-denoising and one-step discretization of the diffusion ODE \citep{songScoreBasedGenerativeModeling2021}.

\subsection{Gaussian Variables}
\label{sct:gaussian}
If $X$ is a multivariate Gaussian distribution, we can diagonalize the covariance matrix of $X$ so that up to a rotation and a translation, $X \sim \N(0, \diag(\tau_1^2,\dots, \tau_n^2))$. Both the denoising and the $W_2$-distance can then be calculated coordinate by coordinate, so we can focus on studying Gaussian variables in $\R$.

If $X \sim \N(0,\tau^2)$, then $Y \sim \N(0,\tau^2+\sigma^2)$ and we can compute $\nabla \log p_Y(y) = \frac{-y}{\tau^2+\sigma^2}$.
It leads to
$$\varphi_\alpha(y) = \frac{\tau^2 + (1-\alpha)\sigma^2}{\tau^2+\sigma^2}y,$$
in particular,
$$
\varphi_1(y) = \frac{\tau^2}{\tau^2+\sigma^2}y, \quad
\varphi_{1/2}(y) = \frac{\tau^2 + (1/2)\sigma^2}{\tau^2+\sigma^2}y.
$$

The denoisers are linear transformations of $Y$, and their laws are Gaussian, given by
\begin{align*}
\varphi_\alpha(Y) &\sim \N\op 0,\frac{(\tau^2 + (1-\alpha)\sigma^2)^2}{\tau^2+\sigma^2}\cp,\\
\varphi_1(Y) &\sim \N\op0,\frac{\tau^4}{\tau^2+\sigma^2}\cp,\\
\varphi_{1/2}(Y) &\sim \N\op0,\frac{(\tau^2 + (1/2)\sigma^2)^2}{\tau^2+\sigma^2}\cp.
\end{align*}
For two Gaussian variables $Z_1 \sim \N(\mu_1,\sigma_1^2)$ and $Z_2 \sim \N(\mu_2,\sigma_2^2)$, the Wasserstein distance is given by \citep[see, e.g.,][]{peyreComputationalOptimalTransport2019}
$$
W_2^2(\L(Z_1),\L(Z_2)) = (\mu_1 - \mu_2)^2 + (\sigma_1 - \sigma_2)^2.
$$
We can therefore compute directly the Wasserstein distances,
\begin{equation}
\label{eq:wasserstein_gaussian}
\begin{array}{rll}
W_2(\L(X),\L(\varphi_\alpha(Y))) &= \Big|\tau - \frac{\tau^2 + (1-\alpha)\sigma^2}{\sqrt{\tau^2 + \sigma^2}}\Big| &= \tau \Big|1 - \frac{1 + (1-\alpha)\op\frac{\sigma}{\tau}\cp^2}{\sqrt{1 + \op\frac{\sigma}{\tau}\cp^2}}\Big|,\\
W_2(\L(X),\L(\varphi_1(Y))) &= \Big|\tau - \frac{\tau^2}{\sqrt{\tau^2 + \sigma^2}}\Big| &= \tau \Big|1 - \frac{1}{\sqrt{1 + \op\frac{\sigma}{\tau}\cp^2}}\Big|,\\
W_2(\L(X),\L(\varphi_{1/2}(Y))) &= \Big|\tau - \frac{\tau^2 + \frac{1}{2}\sigma^2}{\sqrt{\tau^2 + \sigma^2}}\Big| &= \tau \Big|1 - \frac{1 + \frac{1}{2}\op\frac{\sigma}{\tau}\cp^2}{\sqrt{1 + \op\frac{\sigma}{\tau}\cp^2}}\Big|.
\end{array}  
\end{equation}
In particular for $\frac{\sigma}{\tau} \ll 1$ (small noise), with an expansion in $\frac{\sigma}{\tau}$, we get that
$$
W_2(\L(X),\L(\varphi_1(Y))) \sim \frac{1}{2\tau} \sigma^2 \text{ and } W_2(\L(X),\L(\varphi_{1/2}(Y))) \sim \frac{1}{8\tau^3} \sigma^4,
$$
showing that half-denoising beats full-denoising in Wasserstein distance for small noises. In fact, when dividing by~$\tau$ the expression in (\ref{eq:wasserstein_gaussian}), we see that they only depend on the ratio~$\frac{\sigma}{\tau}$. These quantities are plotted in Figure~\ref{fig:wasserstein_gaussian}, where we observe the behavior in $\op\frac{\sigma}{\tau}\cp^2$ for full-denoisng and in $\op\frac{\sigma}{\tau}\cp^4$ for half-denoising, making half-denoising better for small noises, and we remark that it stays better up to $\frac{\sigma}{\tau} = \sqrt{8} \approx 2.83$.

Note, moreover, that for $\alpha = 0$ (no denoising), we have 
$$
W_2(\L(X),\L(\varphi_0(Y))) = W_2(\L(X),\L(Y)) = \tau \left|1 - \sqrt{1 + \op\frac{\sigma}{\tau}\cp^2} \right| \approx \frac{1}{2\tau} \sigma^2,
$$
that is, full-denoising is not better than no denoising at all!

On the contrary, for  $\frac{\sigma}{\tau} \gg 1$ (large noise), we have
$$
W_2(\L(X),\L(\varphi_1(Y))) \approx \tau \text{ and } W_2(\L(X),\L(\varphi_{1/2}(Y))) \approx \frac{1}{2}\sigma,
$$
\sloppy meaning that $W_2(\L(X),\L(\varphi_1(Y))) \ll W_2(\L(X),\L(\varphi_{1/2}(Y)))$. In fact $W_2(\L(X),\L(\varphi_1(Y))) = 0$ for a Dirac measure ($\tau = 0$) whereas $W_2(\L(X),\L(\varphi_{1/2}(Y))) = \frac{1}{2}\sigma$.

Finally, note that we can compute the optimal $\alpha$ for any $\tau$, as $\alpha = 1 + \frac{\tau^2}{\sigma^2} - \frac{\tau}{\sigma}(1 + \frac{\tau^2}{\sigma^2} )^{1/2} = \frac{\tau^2}{\sigma^2} (1 + \frac{\sigma^2}{\tau^2} - (1 + \frac{\sigma^2}{\tau^2})^{1/2})$. For $\frac{\sigma}{\tau} \ll 1$, we get $\alpha = 1/2 + O((\frac{\sigma}{\tau})^2)$ and for $\frac{\sigma}{\tau} \gg 1$, $\alpha = 1 - \frac{\tau}{\sigma}+ O((\frac{\tau}{\sigma})^2)$. We see that the first order terms do not depend on $\tau$, and corresponds either to half-denoising or to full-denoising.

\begin{figure}
    \centering
	\includegraphics[width=\linewidth]{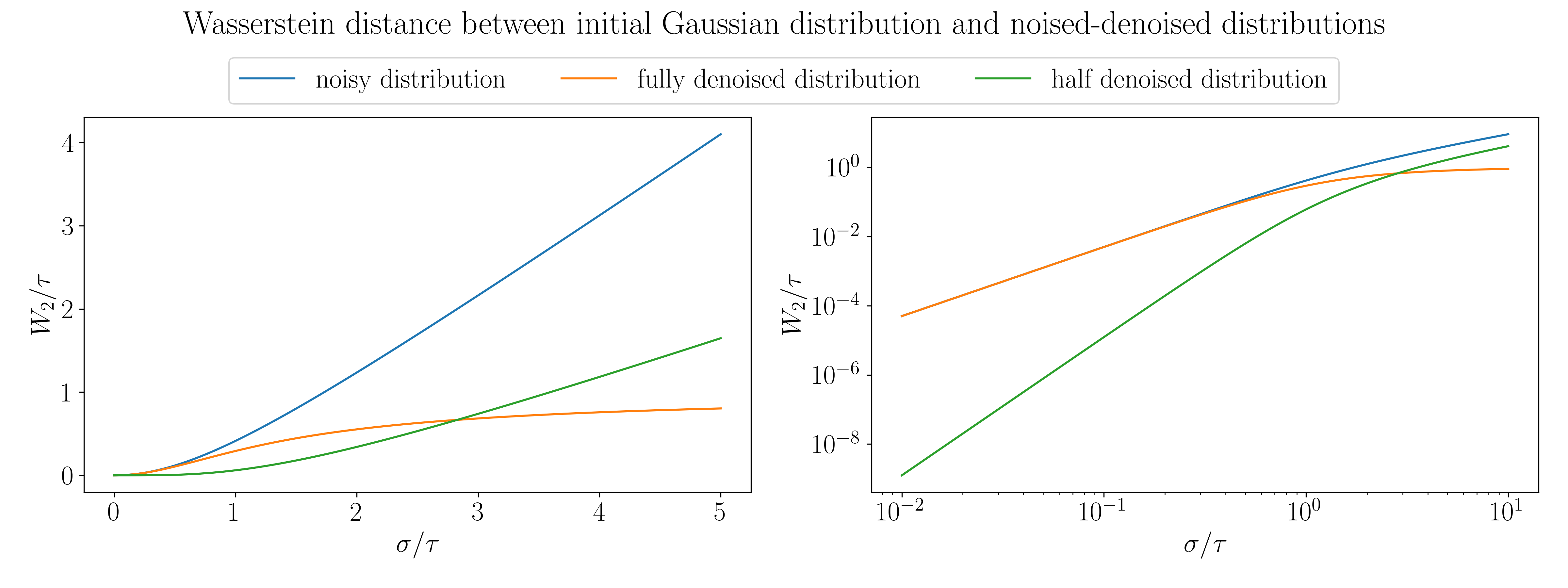}
    \vspace{-.7cm}
	\caption{Wasserstein distances for Gaussian distributions at different levels of noise. Left: linear scale, right: logarithmic scale.}
	\label{fig:wasserstein_gaussian}
\end{figure}

\subsection{Half-Denoising is Better in MMD for Variables with Smooth Densities}
\label{sct:MMD}
\citet{hyvarinenNoisecorrectedLangevinAlgorithm2025} shows that $\hat{p}_{\varphi_{1/2}(Y)}(\xi) = \hat{p}_X(\xi) + O(\sigma^4)$. The proof relies on an uniform bound of the score function $\sup_{x\in \R^d}\Vert \nabla\log p_X(x)\Vert < +\infty$ and does not keep track of all constants.
Based on the same idea, we propose a more complete statement, under a more general $L_2$ assumption, and an application to MMD.

\begin{proposition}
\label{prop:MMD}
Assume that $\E[\Vert \nabla \log p_X(X)\Vert^2]\leq C$. Then for all $\alpha\in\R$,
$$
\forall \xi \in \R^n, \Vert \hat{p}_X(\xi) - \hat{p}_{\varphi_{\alpha}(Y)}(\xi)\Vert \leq  \frac{\sigma^2 (2\alpha\sqrt{C}\Vert\xi\Vert + \Vert\xi\Vert^2)}{2},
$$
and furthermore, for $\alpha = \frac{1}{2}$,
$$
\forall \xi \in \R^n, \Vert \hat{p}_X(\xi) - \hat{p}_{\varphi_{1/2}(Y)}(\xi)\Vert \leq  \frac{\sigma^4 (C\Vert\xi\Vert^2 + \Vert\xi\Vert^4)}{8}.
$$
\end{proposition}
(All proofs can be found in Appendix \ref{sct:proofs}.)

\textbf{Remark:} The hypothesis on the score ($\E[\Vert \nabla \log p_X(X)\Vert^2]\leq C$) \footnote{We can found a similar hypothesis in Assumption 2.5 of \citet{albergoStochasticInterpolantsUnifying2023}. Note that it's also nearly identical to the hypothesis H2 of \citet{confortiKLConvergenceGuarantees2025}, the difference being that the authors take the density with respect to the Gaussian measure rather than the Lebesgue measure. H2 implies $\E[\Vert \nabla \log p_X(X)\Vert^2]\leq C$, but the converse is not true in general.} is natural in the framework of \emph{denoising score matching}, where we use a neural network to learn the score $\log p_Y$. Indeed, \citet{vincentConnectionScoreMatching2011} showed the learning the score with an $L_2$-error $\min_\theta \E[\Vert s_\theta(Y) - \nabla \log p_Y(Y)\Vert^2]$ is equivalent to the denoising objective  $\E [ \Vert X - f_\theta(Y)\Vert^2]$, the latter being used in practice to learn the score. Therefore, as we are learning with an $L_2$-error, it is natural to ask for the $L_2$-bound $\E[\Vert \nabla \log p_Y(Y)\Vert^2]\leq C$. Imposing the bound on $p_X$, $\E[\Vert \nabla \log p_X(X)\Vert^2]\leq C$, allows us to have the bound on $p_Y$ regardless of the level of the added noise $\sigma$ as $\E[\Vert \nabla \log p_X(X)\Vert^2]\leq\E[\Vert \nabla \log p_Y(Y)\Vert^2]$ (Lemma \ref{lemma:expectation_gradient_noise} in Appendix \ref{sct:proofs}). 
It can also be deduced from other hypotheses made in the literature.
For example:
\begin{itemize} 
\begin{samepage}
\item If $X = Z + \ep_0$ , with  $Z \bot \ep_0$ and $\ep_0 \sim \N(0,\tau^2)$,  and $\E[\Vert Z\Vert^2] \leq R^2$ (in particular if $\text{supp}(Z) \subset B(0,R)$ as assumed in Theorem 1 of \cite{saremiChainLogConcaveMarkov2023}), we have  $\E[\Vert \nabla \log p_X(X)\Vert^2]\leq \frac{3(2R^2 + d\tau^2)}{\tau^4}$. Indeed, with Tweedie's formula (\ref{eq:tweedie}),
\end{samepage}
\begin{align*}
\E[\Vert \nabla \log p_X(X)\Vert^2] &= \E\ob\left\Vert\frac{1}{\tau^2}(\E [ Z| X] - X )\right\Vert^2\cb
= \E\ob\left\Vert\frac{1}{\tau^2}(\E [ Z| X] - Z - \ep_0)\right\Vert^2\cb\\
&\leq \frac{3}{\tau^4} (\E[\Vert\E [ Z| X]\Vert^2] + \E[\Vert Z\Vert^2] + \E[\Vert \ep_0\Vert^2]) \quad\text{(Jensen's inequality)}  \\
&\leq \frac{3}{\tau^4} (\E[\E[\Vert Z\Vert^2| X]] + \E[\Vert Z\Vert^2] + d\tau^2)\\
&\quad\text{(Jensen's inequality on conditional expectation)} \\
&= \frac{3}{\tau^4} (2\E[\Vert Z\Vert^2] + d\tau^2) 
\leq \frac{3}{\tau^4} (2R^2 + d\tau^2).
\end{align*}
\item If $X$ is such that $x \mapsto \nabla \log p_X(x)$ is $L$-Lipschitz and $\E[\Vert X\Vert^2] = m_2 < \infty$ (cf. assumptions A1 and A2 of \cite{chenSamplingEasyLearning2023}), we have $\E[\Vert \nabla \log p_X(X)\Vert^2]\leq 2(L^2m_2 + \Vert  \nabla \log p_X(0)\Vert^2)$.
\end{itemize}

These bounds on the characteristic functions lead to bounds in MMD \citep{grettonKernelTwosampleTest2012} for a translation-invariant kernel. 
We assume that we are given a kernel $k$ to compute a MMD, and that $k$ is translation-invariant, i.e., $k(x,y) = \psi(x - y)$, with $\psi: \R^n \rightarrow \R$ a bounded, continuous positive definite function. From Bochner's theorem, there is a unique finite nonnegative Borel measure $\Lambda$ on $\R^n$ such that
$$
\psi(x) = \int e^{-i \xi\cdot x} d\Lambda(\xi).
$$
Then, for two random variables $X$ and $Y$, we have \citep[][Corollary~4]{sriperumbudurHilbertSpaceEmbeddings2010}
$$
\MMD_k(\L(X),\L(Y)) = \op \int |\hat{p}_X(\xi) - \hat{p}_Y(\xi)|^2 d\Lambda(\xi) \cp^{1/2}.
$$

\begin{corollary}
\label{cor:MMD}
Assume that, $\E[\Vert \nabla \log p_X(X)\Vert^2]\leq C$, and, 
$$
C_2 = \int \Vert\xi\Vert^2 d\Lambda(\xi) < \infty,
\quad\quad
C_4 = \int \Vert\xi\Vert^4 d\Lambda(\xi) < \infty,
\quad\text{and }\quad
C_8 = \int \Vert\xi\Vert^8 d\Lambda(\xi) < \infty.
$$
Then, for all $\alpha \in \R$,
$$
\MMD_k(\L(X),\L(\varphi_{\alpha}(Y))) \leq K_1 \sigma^2,
$$
with $K_1=  \frac{\sqrt{4\alpha^2C C_2 + C_4}}{\sqrt{2}}$.

Furthermore, for $\alpha = \frac{1}{2}$,
$$
\MMD_k(\L(X),\L(\varphi_{1/2}(Y))) \leq K_2 \sigma^4,
$$
with $K_2 = \frac{\sqrt{C^2 C_4 + C_8}}{4\sqrt{2}}.$
\end{corollary}

The first result applies for $\alpha = 1$, hence we can compare the bound that we get for full-denoising (first result) to the one we get for half-denoising (second result).
It shows that for regular enough densities ($\E[\Vert \nabla \log p_X(X)\Vert^2]\leq C$), we have $\MMD_k(\L(X),\L(\varphi_{1/2}(Y))) = O(\sigma^4)$, whereas $\MMD_k(\L(X),\L(\varphi_{1}(Y))) = O(\sigma^2)$ and hence the bound on $\MMD_k(\L(X),\L(\varphi_{1/2}(Y)))$ is negligible compared to the bound on $\MMD_k(\L(X),\L(\varphi_{1}(Y)))$ for small $\sigma$, therefore extending the result seen above for Gaussian distributions.
Moreover, the bound $\MMD_k(\L(X),\L(\varphi_{\alpha}(Y))) = O(\sigma^2)$ also applies for $\alpha = 0$, hence as in the Gaussian case, full-denoising does not do better than no-denoising for small $\sigma$.

\subsection{Half-Denoising is Better in Wasserstein-2 for Variables with Smooth Densities}
\label{sct:wasserstein}
We now prove similar bounds in Wasserstein distance. To do so, we introduce a continuous diffusion process, progressively adding Gaussian noise to $X$ with a Brownian motion, and the diffusion ODE, which generates the same marginals with a deterministic process.
This deterministic process, sometimes referred to as the probability flow ODE \citep{songScoreBasedGenerativeModeling2021}, has been used as a way to have deterministic generation with diffusion models. Here, we will use the fact that half-denoising can be seen as a one step discretization of this ODE. We give a complete proof of the construction of this ODE in Appendix \ref{sct:fokker}. Here is a brief overview.

We define a process $X_t = X + B_t$, with $B_t$ a Brownian motion, such that we have $X = X_0$ and $Y = X_{\sigma^2}$. We also denote $p_t = p_{X_t}$ the density of $X_t$ with respect to the Lebesgue measure. $p_t$ verifies the Fokker-Planck equation $\partial p_t = \Delta p_t$, which can be rewritten as $\partial p_t = - \nabla\cdot (- p_t\nabla\log p_t)$. We deduce that we can then defined the following ODE:
\begin{equation}
\label{eq:ODE}
\left\{\begin{array}{rll}
    \frac{d x_t}{dt} &= -\frac{1}{2} \nabla \log p_t(x_t) &\text{for } t > 0\\
    x_{t^*} &= X_{t^*} &\text{for some } t^* > 0,
\end{array}\right.
\end{equation}
and that it has the same marginals as $X_t$:  $\forall t \in [0,+\infty[, x_t \sim X_t$, and verifies, for all $t,s \geq 0$,
$$
x_t - x_s = -\frac{1}{2}\int_s^t \nabla \log p_u(x_u) du. 
$$
Note here that $\nabla \log p_t$ is not, in general, Lipschitz-continuous near $t=0$.\footnote{This should not come as a surprise, as if we take $\mu_X$ to be a Dirac at $0$, then all trajectories will coincide at $t=0$, which is prohibited by Cauchy-Lipschitz Theorem.}
That's why we take an initial condition\footnote{The choice of $t^*$ does not matter as for any $t^*>0$, we will have the same marginals for $(x_t)_{t\geq0}$, and in particular it will give a path between $x_{\sigma^2}\sim Y$ and  $x_{0}\sim X$.} at $t^*>0$. However, we verify that the trajectory can in fact be integrated up to $t=0$ (see Appendix \ref{sct:fokker} for more details).

As \citet{gentiloni-silveriLogConcavityScoreRegularity2025},\footnote{\citet{gentiloni-silveriLogConcavityScoreRegularity2025} compute Wasserstein distance for the diffusion SDE with a multiple step discretization.} we use the continuous time process $x_t$, and its one step discretization, to get natural couplings between distribution of $X$ and $\varphi_\alpha(Y)$ and compute Wasserstein distances. We have $X \sim x_0$ and $\varphi_\alpha(Y) \sim\hat{x}_0$, with $\hat{x}_0 = x_t + \alpha t \nabla \log p_t(x_t)$ and $t=\sigma^2$. This leads to the bound $W_2^2(\L(X),\L(\varphi_\alpha(Y))) \leq \E[\Vert x_0 - \hat{x}_0\Vert^2]$, with
$$
x_0 - \hat{x}_0 = \frac{1}{2}\int_0^t \nabla \log p_s(x_s) ds - \alpha t \nabla \log p_t(x_t).
$$
To conclude, we need to be able to bound $\nabla\log p_s(x_s)$ for $s \in [0,t]$, and to do that, we only need to have a bound on $\nabla \log p_X(X)$. Formally, we have the following result:
\begin{proposition}
\label{prop:W2}
Assume that $\E[\Vert \nabla \log p_X(X)\Vert^2]\leq C$. Then
$$
W_2 (\L(X),\L(\varphi_\alpha(Y))) \leq \sqrt{\frac{(1+4\alpha^2)C}{2}} \sigma^2.
$$
\end{proposition}

For $\alpha = 1$, this bound is already better (in $\sigma$) than the bound given by Proposition 2 of \citet{saremiChainLogConcaveMarkov2023} which is $W_2(\L(X),\L(\varphi_1(Y))) = O(\sigma)$. Moreover this bound is tight (in $\sigma$) as for a Gaussian variable with variance~$\tau^2$, we have $W_2(\L(X),\L(\varphi_1(Y))) \sim \frac{1}{2\tau}\sigma^2$ for small enough~$\sigma$.
Note that if we remove the assumption (i.e., $C = +\infty$), then we fall back on the bound by \citet{saremiChainLogConcaveMarkov2023} (take for example $X$ a Dirac measure, for which we have seen that $ W_2(\L(X),\L(\varphi_{1/2}(Y))) = \frac{1}{2}\sigma$).

But for $\alpha = \frac{1}{2}$ we can hope to have a better bound, in $\sigma^4$, as it was the case with the MMD. Indeed, we have
\begin{align*}
x_0 - \hat{x}_0 
&= \frac{1}{2}\int_0^t \nabla \log p_s(x_s) ds - \frac{1}{2} t \nabla \log p_t(x_t) \\
&= \frac{1}{2}\int_0^t \op\nabla \log p_s(x_s) - \nabla \log p_t(x_t)\cp ds \\
&= \frac{1}{2}\int_0^t \int_s^t \frac{d}{du}\op\nabla \log p_u(x_u) \cp du ds.
\end{align*}
If we can control $\E[\Vert \frac{d}{dt}\nabla \log p_{t}(x_{t}) \Vert^2] \leq C$, we will get $\E[\Vert x_0 - x_t\Vert^2] = O(t^4)$ hence $W_2(\L(X),\L(\varphi_{1/2}(Y))) = O(\sigma^4)$. Formally, we need the following lemma:

\begin{lemma}
\label{lemma:derivative_bound}
Assume that the variable $X$ of density $p_X$ satisfies:
\begin{itemize}
    \item $\log p_X \in \mathcal{C}^3(\R^d)$. 
    \item $C_1 = \E[\Vert \nabla \log p_X(X)\Vert^6] < \infty$.
    \item $C_2 = \E[\Vert \nabla^2\log p_X(X)\Vert_{\textnormal{op}}^3] < \infty$, where $\Vert A\Vert_\textnormal{op}$ is defined for any matrix $A$ as $\Vert A \Vert_\textnormal{op} = \sup_{x\neq 0}\frac{\Vert Ax\Vert}{\Vert x \Vert}$.
    \item $C_3 = \E[\Vert \nabla \Delta \log p_X(X)\Vert^2] < \infty$.
\end{itemize}
Then,
$$
\E\ob\left\Vert \frac{d}{dt}\nabla \log p_t(x_t)\right\Vert^2\cb \leq C,
$$
with $C = \frac{9}{4}(4C_1 + (2d^2+5)C_1^{1/3}C_2^{2/3}+C_3)$.
\end{lemma}

Using this lemma, we have the following proposition:
\begin{proposition}
\label{prop:W2half}
Under the assumptions of Lemma \ref{lemma:derivative_bound}, we have 
$$
W_2(\L(X),\L(\varphi_{1/2}(Y))) \leq K \sigma^4.
$$
with $K = \frac{\sqrt{3}}{4}\sqrt{4C_1 + (2d^2+5)C_1^{1/3}C_2^{2/3}+C_3}$.
\end{proposition}

\textbf{Remarks}:
\begin{itemize}
\begin{samepage}
    \item In Appendix \ref{sct:usual_distrib}, we show that if $X = Z + \ep_0$ , with  $\E[\Vert Z\Vert^6] <\infty$, $Z \bot \ep_0$ and $\ep_0 \sim \N(0,\tau^2)$, then $X$ verifies the assumptions of Lemma \ref{lemma:derivative_bound}. In particular, it applies to the case of $\text{supp}(Z) \subset B(0,R)$, assumed in Theorem 1 of \citet{saremiChainLogConcaveMarkov2023}, as then $\E[\Vert Z\Vert^6] \leq R^6$. It also proves that a mixture of Gaussian distributions verifies the hypothesis (take $\tau$ the smallest eigenvalue of all the covariances matrices of the Gaussian distributions in the mixture).
\end{samepage}
    \item The assumptions of Lemma \ref{lemma:derivative_bound} and Proposition \ref{prop:W2half} control the regularity of the density $p_X$. The fact that we need to bound derivative up to order 3 comes directly from the Fokker-Planck equation $\partial_tp_t = \Delta p_t$, which can be interpreted heuristically as ``one derivative in $t$ equals two derivatives in $x$'', hence to control $\frac{d}{dt}\nabla \log p_t(x_t)$, we need to control derivatives in $x$ up to order 3.
    \item The control of the Hessian $\Vert \nabla^2\log p_X(x)\Vert_{\textnormal{op}}$ corresponds to controlling the Lipschitz constant of the function $ x \mapsto  \nabla\log p_X(x)$ and is an assumption usually done in the literature \citep[see, e.g.,][assumption A1]{chenSamplingEasyLearning2023}. Having an uniform bound is a strong assumption, and in particular it implies that the distribution has full support on $\R^d$. But here we only need to control this quantity in expectation under the law of~$X$. As a consequence, this result can apply to distribution such that $ x \mapsto  \nabla\log p_X(x)$ is not globally Lipschitz-continuous, and moreover it does not even have to be defined everywhere. Take for example
    $$
    p_X(x) = \left\{\begin{array}{cl}
        \frac{1}{Z} e^{-\frac{1}{1-x^2}}& \text{if } |x| \leq 1\\
        0& \text{otherwise},
    \end{array}
    \right.
    $$
    where $Z$ is a normalizing constant (Fig. \ref{fig:density_compact}). Then $\log p_X(x)= -\frac{1}{1-x^2}$ is only defined on the interval $(-1,1)$ and we have
    \begin{align*}
        \nabla \log p_X(x) &= -\frac{2x}{\op1-x^2\cp^2}, \\
        \nabla^2 \log p_X(x) &= -\frac{2\op3x^2+1\cp}{\op1-x^2\cp^3},\\
        \nabla\Delta \log p_X(x) &= -\frac{24x\op x^2+1\cp}{\op1-x^2\cp^4},
    \end{align*}
    with none of these quantities being bounded on $(-1,1)$. However, our assumptions only require the bounds in expectation, and as the density $p_X$ decreases exponentially fast when $|x|\rightarrow 1$, overcoming the polynomial growth of the derivatives, all three constants $C_1$, $C_2$ and $C_3$ are finite.
    
    Note that as Propositions \ref{prop:W2} and \ref{prop:W2half} can be applied to distributions with compact support, they can be combined with Proposition \ref{prop:mixt} of section \ref{sct:mixt}.
    \item The assumptions are not verified for a very singular distribution, e.g., a Dirac $\mu_X = \delta_0$, for which $\log p_X$ is not even defined. In this case, we have $x_t \sim \N(0,tI)$ hence $\log p_t(x) = -\frac{\Vert x\Vert^2}{2t}$, $\nabla\log p_t(x) = -\frac{x}{t}$ and $\E[\Vert\nabla\log p_t(X) \Vert^2_\text{op}] = \frac{d}{t}$, which is not bounded, and not integrable near $0$. The result of Proposition \ref{prop:W2half} don't apply, and indeed we have (cf. section \ref{sct:gaussian}) $W_2(\L(X),\L(\varphi_{1/2}(Y))) = \frac{\sigma}{2}$ whereas ${W_2(\L(X),\L(\varphi_{1}(Y))) = 0}$.
    \item All results from this subsection can be extended to Wasserstein-$p$ distance for any $p \geq 1$ (see Appendix \ref{sct:extended_wasserstein}).
\end{itemize}

\begin{figure}
    \centering
	\includegraphics[width=.5\linewidth]{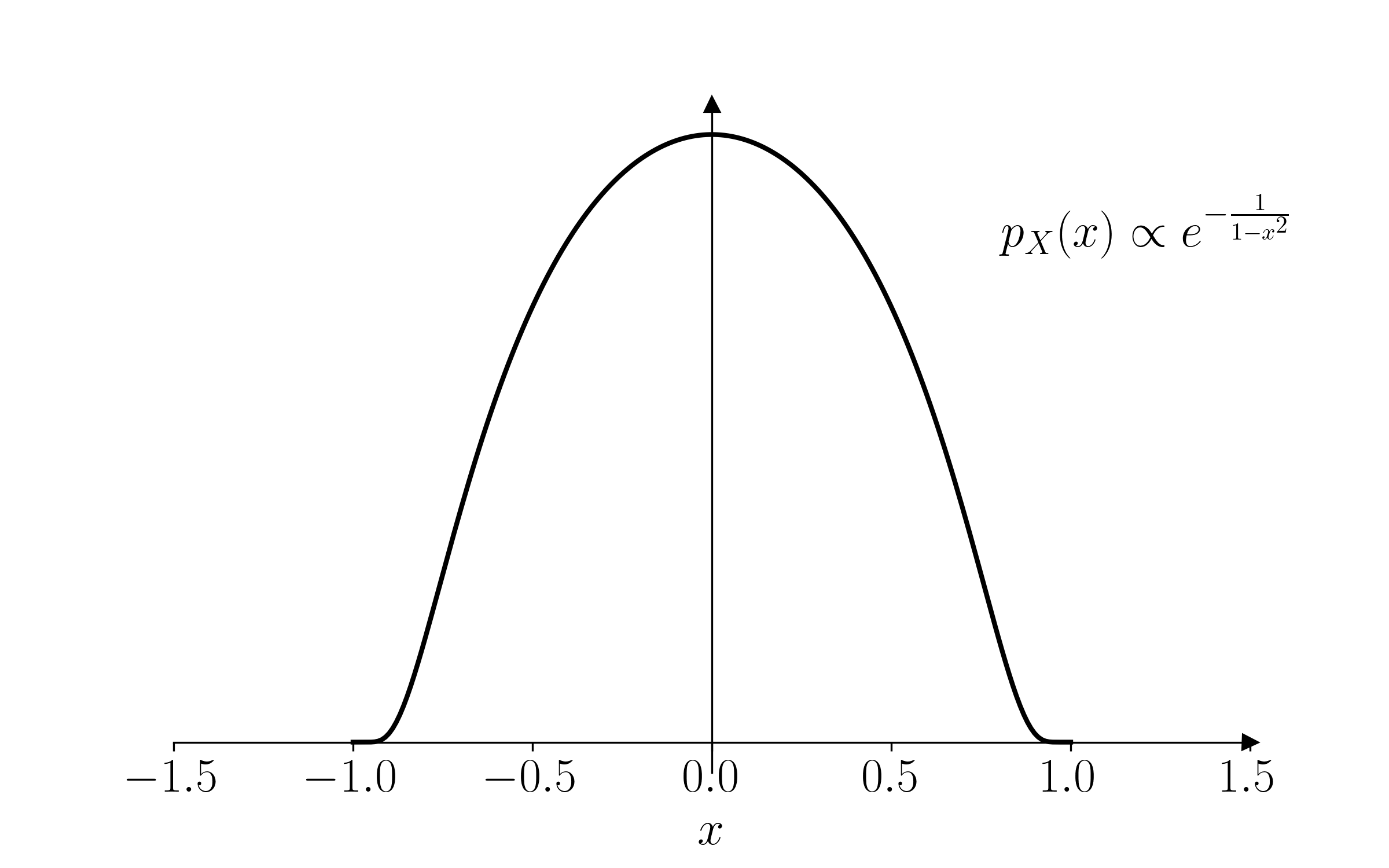}
	\caption{Example of a smooth density with a compact support.}
	\label{fig:density_compact}
\end{figure}

\section{Variable with Support on a Subspace: Balancing Between the Impacts of Singularity and Regularity of the Distribution}
\label{sct:subspace}

In many practical applications, the data distribution is supported on a lower dimensional manifold, a case known as the \emph{manifold hypothesis} \citep[see, e.g.,][]{tenenbaumGlobalGeometricFramework2000,bengioRepresentationLearningReview2013,feffermanTestingManifoldHypothesis2016}. 
Locally, this manifold will look like a linear space, therefore, we take a look at the idealized case when the variable is supported on a \emph{linear} lower-dimensional subspace.

\begin{proposition}
\label{prop:subspace}
Assume that $X$ is supported on a linear subspace $H$ of dimension $m$, with $m\leq d$. Write $X_1 = p_H(X)$, with $p_H$ the orthogonal projection on $H$, and $Y_1 = X_1 + \ep_1 \in H$ with  $\ep_1 \sim \N(0,\sigma^2I_m) \in H$. Then:
$$
W_2^2(\L(X),\L(\varphi_\alpha(Y))) = W_2^2(\L(X_1),\L(\varphi_\alpha(Y_1))) + (d-m)(1-\alpha)^2\sigma^2.
$$ 
\end{proposition}

We can interpret this result as trade-off between $\alpha = 1$ (full-denoising) which cancels the second term, as it ensures that $\varphi_\alpha(Y)$ belongs to the subspace $H$ (see the proof in Appendix~\ref{sct:proofs} for more details), and  $\alpha = 1/2$ (half-denoising) that will reduce the first term if $p_{X_1}$ is regular enough to apply the results of the previous section.
More precisely, under the assumption that the density on the subspace is regular enough, $W_2^2(\L(X_1),\L(\varphi_\alpha(Y_1)))$ is in $O(\sigma^4)$ for full-denoising (Proposition~\ref{prop:W2}), and in $O(\sigma^8)$ for half-denoising (Proposition~\ref{prop:W2half}). Half-denoising is better on the subspace for $\sigma$ small enough, but the term $\frac{(d-m)\sigma^2}{4}$ dominates as $\sigma$ goes to zero.
Depending on the ratio between the dimension $m$ of the subspace and $d$ of the whole space, there  may or may not be a sweet spot for half-denoising where the gain obtained by reducing the first term outweighs the increase in the second term. 

We illustrated this in Figure~\ref{fig:wasserstein_gaussian_subspace}, where the target distribution is a Gaussian $\N(0,\tau^2 I_m)$ supported on the subspace $\R^m\times\{0\}^{d-m}$ (we use the closed-form formulas from Section~\ref{sct:gaussian}).
For full-denoising, the only term that remains is $ W_2^2(\L(X_1),\L(\varphi_\alpha(Y_1)))$, that is plotted in orange.
Half-denoising is plotted in green, and is the sum of the term $W_2^2(\L(X_1),\L(\varphi_\alpha(Y_1)))$ (red dotted line) and the term $\frac{(d-m)\sigma^2}{4}$ (purple dotted line).
The term corresponding to half-denoising on the subspace (red dotted line) is smaller that the Wasserstein distance for full-denoising (orange line) for $\sigma \lesssim 2.83\tau$.
However, we observe that as $\sigma$ goes to zero, the term $\frac{(d-m)\sigma^2}{4}$ (purple dotted line) dominates hence the Wasserstein distance for half-denoising (green line) is greater that the error for full-denoising.
For $d=10,m=9$ (a), there is a sweet spot in which the trade-off is in favor of half-denoising, while for $d=10,m=5$ (b), full-denoising is better at all noise levels.

\begin{figure}[ht]
    \captionsetup[subfloat]{captionskip=-10pt}
    \centering
	\subfloat[][$d=10, m = 9$]{
    \includegraphics[width=.98\linewidth]{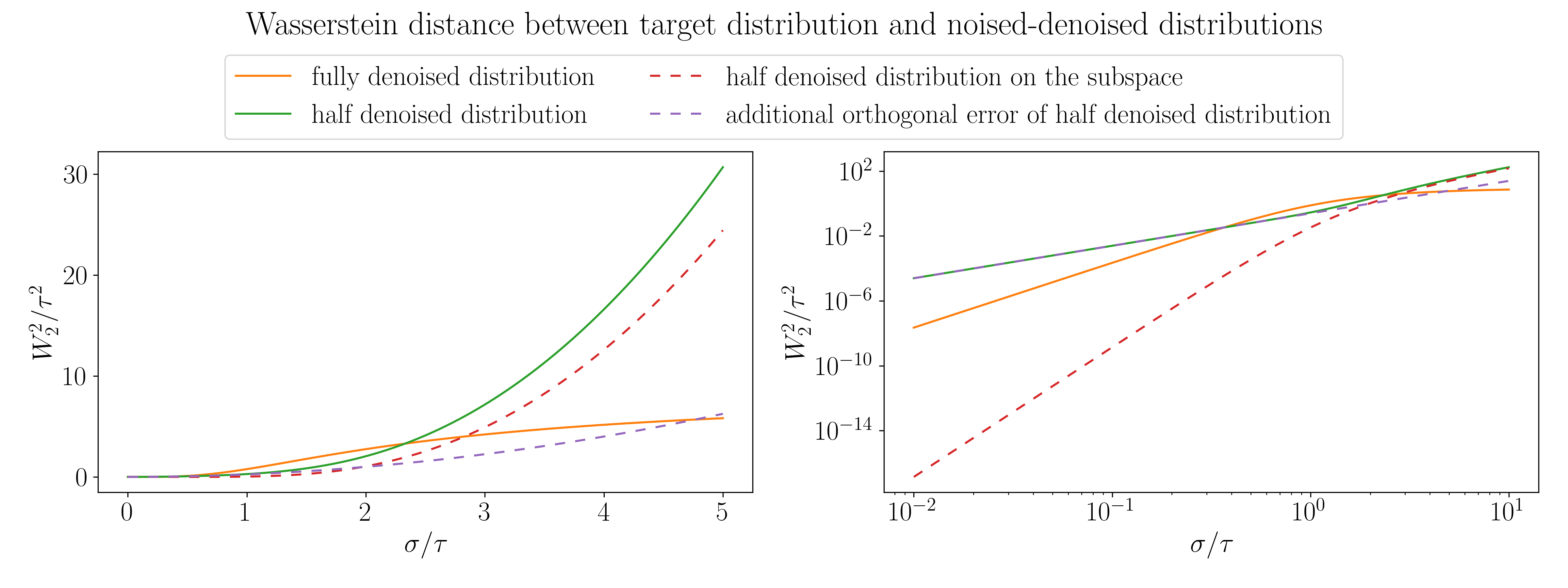}
    }
    \vspace{10pt}
    \subfloat[][$d=10, m = 5$]{
    \includegraphics[width=.98\linewidth]{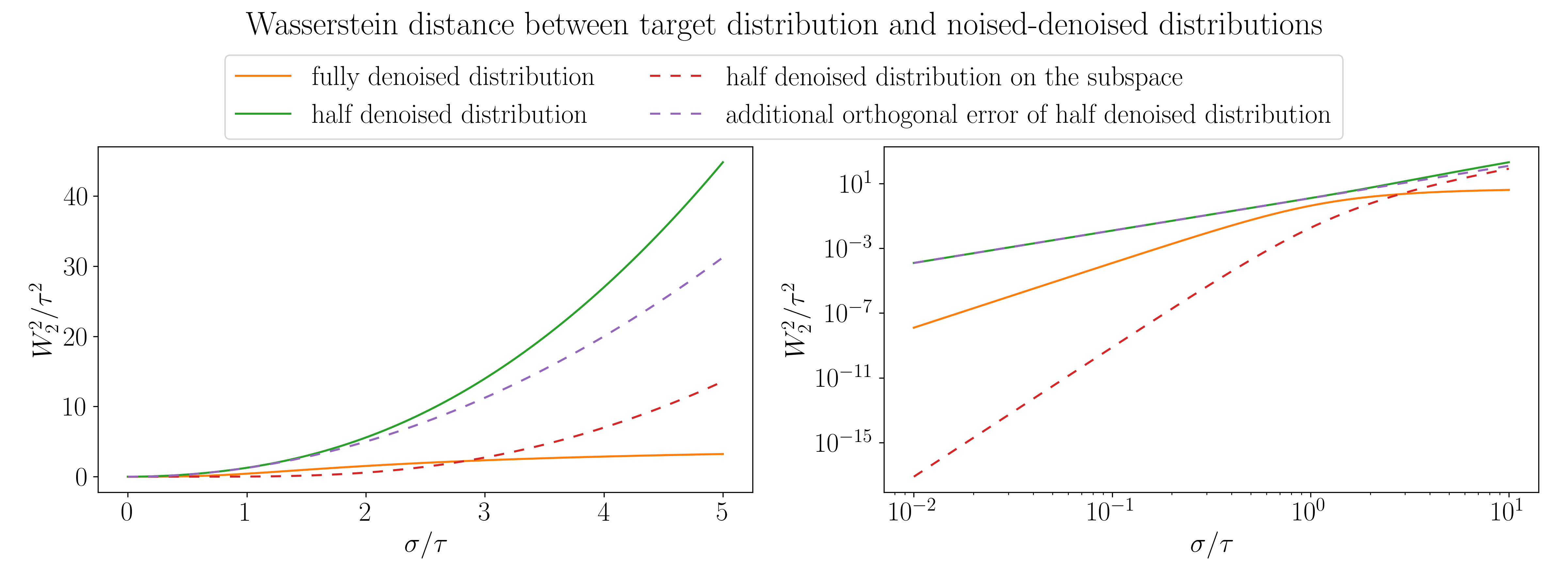}
    }
    \caption{Wasserstein distances for Gaussian distribution supported on the subspace $\R^m\times\{0\}^{d-m}$ as target distribution and different noise levels $\sigma$.}
	\label{fig:wasserstein_gaussian_subspace}
\end{figure}

Proposition~\ref{prop:subspace} also shows that full-denoising is adaptive to the low dimensional structure of the data, as the Wasserstein distance only depends on the distance between distributions on the lower-dimensional subspace. Therefore, in the case where $m \ll d$, it alleviates the curse of dimensionality.

\section{Mixtures of Distributions with Disjoint Compact Supports Behave like Independent Variables}
\label{sct:mixt}

In the case of the \emph{manifold hypothesis}, it is also common to suppose that the distribution is made of small pockets with high density of probability, representing different classes of objects, separated by regions of low density.
We model this case be saying that the distribution of $X$ is a mixture of distributions with disjoint compact supports.
In this case, we show that the denoising performance for a mixture of distributions with disjoint compact supports behaves as if we were denoising each variable independently, plus an exponentially decreasing term.

More formally, let $ X \sim \mu_X = \sum_{i=1}^N \pi_i \mu_i$ (with  $\sum_{i=1}^N \pi_i = 1,\pi_i \geq 0$ ) a mixture of distribution $\mu_i$ with compact support $S_i$ such that $D = \min_{i\neq j} d(S_i,S_j) > 0$ (with $d(S_i,S_j) = \min_{x_i \in S_i, x_j \in S_j}\Vert x_i - x_j\Vert$).
We denote $Y = X + \ep$ with $X \bot \ep$ and $\ep \sim \N(0, \sigma^2 I)$, for $\sigma>0$. For $\alpha \in \R$, we denote $\varphi_\alpha(y) = y +\alpha \sigma^2 \nabla \log p_Y(y)$, $\nu$ the law of $Y$ and $\mu_\alpha$ the law of $\varphi_\alpha(Y)$. Similarly, for $X_i \sim \mu_i$, and  $Y_i = X_i + \ep$ with $X_i \bot \ep$ and $\ep \sim \N(0, \sigma^2 I)$,  we denote $\varphi_{i,\alpha}(y) = y +\alpha \sigma^2 \nabla \log p_{Y_i}(y)$, $\nu_i$ the law of $Y_i$ and $\mu_{i,\alpha}$ the law of $\varphi_{i,\alpha}(Y_i)$.
We have the following proposition:
\begin{proposition}
\label{prop:mixt}
\begin{samepage}
We have, under the above assumptions,
$$W_2^2(\mu_X,\mu_\alpha) \leq 2 \sum_i \pi_i W_2^2(\mu_i,\mu_{i,\alpha}) + O\op\frac{1}{\sigma^{\max(d-2,8)}}\exp\op-\frac{K}{\sigma^2}\cp\cp,$$
where $K$ is a constant that depends only on $D$.
\end{samepage}
\end{proposition}
The $O$ hides a constant depending on $\alpha$, $d$, $N$ and $R$ such that $\text{supp}(X) \subset B(0,R)$ (see the proof in Appendix \ref{sct:proofs} for more details).

In the case where the $\mu_i$'s are Dirac measures, and for $\alpha = 1$, we have $ W_2^2(\mu_i,\mu_{i,\alpha}) = 0$, hence $$W_2^2(\mu_X,\mu_\alpha) = O\op\frac{1}{\sigma^{\max(d-2,8)}}\exp\op-\frac{K}{\sigma^2}\cp\cp,$$ which is way better than polynomial rates in $\sigma$ seen above.

In Figure~\ref{fig:wasserstein_mixture}, we illustrate this result for a mixture of two Dirac measures in 1D, with ${X = \frac{\delta_{-\mu} + \delta_\mu}{2}}$ for some $\mu >0$. In this case, we can derive integral expressions for $\frac{W_2 (\L(X),\L(\varphi_\alpha(Y)))}{\mu}$ as functions of $\frac{\sigma}{\mu}$ that can be evaluated numerically (see Appendix~\ref{sct:integral_expression_mixture}). For full-denoising, the only remaining term is the $O\op\frac{1}{\sigma^{\max(d-2,8)}}\exp\op-\frac{K}{\sigma^2}\cp\cp$, which decreases much faster than the error for half-denoising.

\begin{figure}
    \centering
	\includegraphics[width=\linewidth]{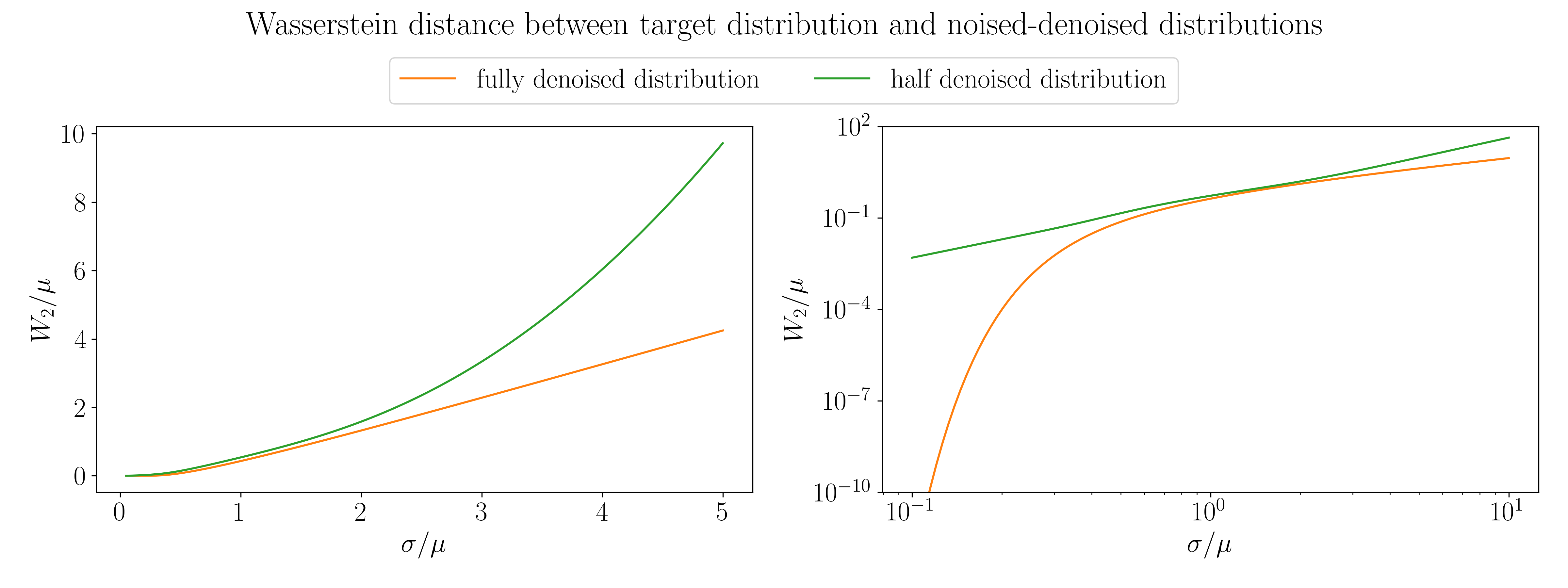}
    \vspace{-.7cm}
	\caption{Wasserstein distances for a mixture of two Dirac measures $\frac{\delta_{-\mu} + \delta_\mu}{2}$ as target distribution and different noise levels $\sigma$.}
	\label{fig:wasserstein_mixture}
\end{figure}

\section{Consequences}
In this section, we examine the consequences of our results for methods based on one-step denoising as well as for multi-step diffusion models.

\subsection{Linear Manifold Hypothesis}
We can combine Propositions \ref{prop:subspace} and \ref{prop:mixt} to tackle what we call the \emph{linear manifold hypothesis}.
We defined the linear  manifold hypothesis as a simplified version of the manifold hypothesis, where the data distribution is supported on disjoint compact sets, each of these belonging to a (different) linear subspace of low dimension, as illustrated in Figure \ref{fig:manifold}. 
Then applying Proposition \ref{prop:mixt} allows to bound the Wasserstein distance between the original distribution and the fully denoised distribution by the sum of the Wasserstein distances between the distributions on each compact sets (plus on exponentially decreasing term). As each compact set belongs to a low-dimensional subspace, we can apply Proposition \ref{prop:subspace}, which tells that for full-denoising, the Wasserstein distance depends only on the distribution on the sub-space. In this case, we see that full-denoising can alleviate the curse of dimensionality even if the support of the mixture distribution itself spans the whole space as it adapts to the local linear structure of the distribution.
More generally, understanding the performance of score-based generative models for data distribution supported on a low dimensional manifold is an active area of research \citep[see, e.g.,][]{tangAdaptivityDiffusionModels2024,azangulovConvergenceDiffusionModels2025}. 

In particular, a result from \citet{azangulovConvergenceDiffusionModels2025} also gives a bound for full-denoising that only depends on the manifold dimension under the manifold hypothesis.
For a smooth manifold $M$ of dimension $d_M < d$ and a regular enough density $p_X$  (see the original paper for all assumptions), we can deduce from their Theorem 12 that
\begin{equation}
\label{eq:bound_azangulovConvergenceDiffusionModels2025}
W_2(\L(X),\L(\varphi_1(Y))) \leq \sigma \sqrt{8(40d_M \log_+\sigma^{-1} + 8d_M C_{\log}+ 3)}
\end{equation}
where $\log_+ : x\mapsto \max (\log x,1)$ and $C_{\log}$ is a constant that controls the regularity of $p_X$, in particular, it must verify $\forall x \in M, e^{-d_MC_{\log}}< p_X(x) < e^{d_MC_{\log}}$. 
This bound shows that the Wasserstein distance between the fully denoised distribution and the original distribution only depends on the subspace dimension. 
Compared to Propositions \ref{prop:subspace} and \ref{prop:mixt}, it allows to tackle the more complete case of a smooth manifold, however it relies on stronger smoothness assumptions, with in particular the need for a lower and upper bounded density  ${e^{-d_MC_{\log}}< p_X(x) < e^{d_MC_{\log}}}$, while our results do not not need any regularity assumptions on $p_X$.
Moreover, we see that (\ref{eq:bound_azangulovConvergenceDiffusionModels2025}) scales as $O(\sigma)$, whereas Proposition~\ref{prop:W2} gives a scaling in $O(\sigma^2)$ is the density is regular enough.
Future work could be dedicated to seeing whether we can draw inspiration from their approach and ours to achieve a more general result with a scaling in $O(\sigma^2)$.

\begin{figure}[ht]
    \centering
	\subfloat[][]{
    \includegraphics[width=.32\linewidth]{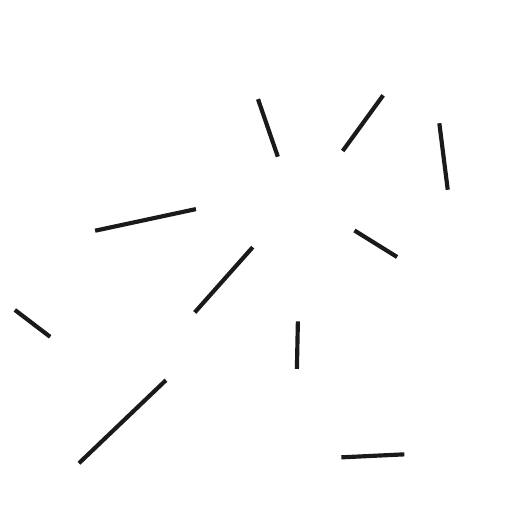}
    }
    \hspace{1cm}
    \subfloat[][]{
    \includegraphics[width=.32\linewidth]{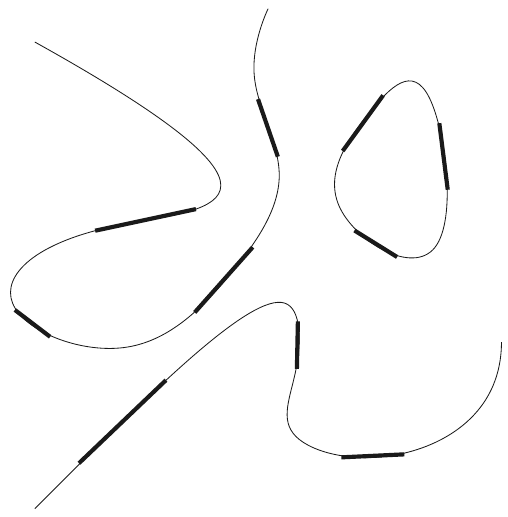}
    }
    \caption{Illustration of the linear manifold hypothesis (a -- thick line), and the manifold hypothesis (b -- fine line).}
	\label{fig:manifold}
\end{figure}

\subsection{Half vs Full for One-Step Denoising}
Our results give insights to choose between half and full denoising in algorithms that use one step denoising. This includes the algorithm proposed by \citet{hyvarinenNoisecorrectedLangevinAlgorithm2025}, which can be viewed (for $\mu = \frac{\sigma^2}{2}$) as a regular Langevin on the noisy variable $Y$ followed by one step of half-denoising, \emph{walk-jump sampling} \citep{saremiNeuralEmpiricalBayes2019}, which is a regular Langevin on the noisy variable $Y$ followed one step of full-denoising,  or also \emph{Sequential multimeasurement walk-jump sampling} \citep{saremiChainLogConcaveMarkov2023}, where noise is first reduced by averaging multiple noisy measurements then one step of full-denoising is applied.
While in the algorithms listed here, the choice of half or full denoising is arbitrary, our results provide guidance on choosing one over the other depending on practical cases. If the density is regular, then half-denoising is preferable, whereas if it is singular, for example in the case of the manifold hypothesis, full-denoising is preferable.  

\subsection{Denoising Diffusion Models}
We will briefly discuss the first insights that our results on one step denoising give about multi-step diffusion models, leaving a further analysis for future work. We will focus on diffusion models with a deterministic denoising process, such as the probability flow ODE by \citet{songScoreBasedGenerativeModeling2021} and DDIM by \citet{songDenoisingDiffusionImplicit2021}.
We start from the more general formulation of the probability flow ODE given by \citet{karrasElucidatingDesignSpace2022}:
\begin{equation}
\label{eq:ODE_karras}
\frac{dx_t}{dt} = \frac{\dot s(t)}{s(t)}x_t - s(t) \dot \sigma(t) \sigma(t)\nabla\log p\op\frac{x_t}{s(t)}; \sigma(t)^2\cp,
\end{equation}
where $p\op x; \sigma^2\cp$ is the density
\footnote{We differ slightly from the notation of \citet{karrasElucidatingDesignSpace2022} as we parametrize this density by $\sigma^2$ rather than $\sigma$, to be more coherent with the usual notation for Gaussian distributions and such that we have $p_t(x) = p(x;t)$ for the process of (\ref{eq:ODE}).}
of the variable $X+\N(0,\sigma^2)$, and $x_t$ as marginal $x_t \sim s(t)\cdot(X + \N(0,\sigma(t)^2))$.

From this equation, we get algorithms by choosing some time $T$, sampling $\hat{X}_0 \sim \N(0,s(t)^2\sigma(t)^2)\approx x_T$,  and discretizing the process given by (\ref{eq:ODE_karras}) with $N$ (possibly non uniform) steps $t_0 = T > t_1 > \cdots > t_N = 0$.

There are multiple ways to discretize (\ref{eq:ODE_karras}), leading to different algorithms (see Appendix~\ref{sct:derivation_diffusion_models} for more details on the derivations).  The DDIM algorithm \citep{songDenoisingDiffusionImplicit2021} corresponds to the update
$$
\hat X_{k+1} = \frac{s(t_{k+1})}{s(t_k)}\hat X_k + s(t_{k+1})(\sigma(t_k)^2- \sigma(t_k)\sigma(t_{k+1})) \nabla\log p\op\frac{\hat X_k}{s(t_k)}; \sigma(t_k)^2\cp.
$$
Equivalently,\footnote{From a theoretical perspective, looking at $\hat X_k$ and $\td X_k$ is strictly equivalent as we can go from one to the other simply by multiplying by a known quantity. Note however that this scaling can have consequences on the training and numerical stability of the algorithms (but those effects can be dissociated by adding multiplicative constants to modify how our neural network and training loss are parametrized, as done by \citet{karrasElucidatingDesignSpace2022}).} using the rescaled variable $\td X_k = \hat X_k / s(t_k)$, we have
$$
\td X_{k+1} = \td X_k +\alpha_k\op\sigma(t_k)^2- \sigma(t_{k+1})^2\cp \nabla\log p\op\td X_k; \sigma(t_k)^2\cp,
$$
with $\alpha_k = \frac{\sigma(t_k)}{\sigma(t_{k+1}) + \sigma(t_k)}$. This update corresponds to $\alpha_k$-denoising, between noise levels $\sigma(t_k)$ and $\sigma(t_{k+1})$.

On the other hand, the Euler discretization of (\ref{eq:ODE_karras}) gives
$$
\hat X_{k+1} = \op 1 + \frac{\dot s(t_{k})(t_{k+1}-t_k)}{s(t_k)}\cp\hat X_k - s(t_{k}) (t_{k+1}-t_k) \dot \sigma(t_k) \sigma(t_k) \nabla\log p\op\frac{\hat X_k}{s(t_k)}; \sigma(t_k)^2\cp,
$$
and with the rescaled variable $\td X_k = \hat X_k / s(t_k)$, we have
\begin{equation}
\label{eq:general_update_multisteps}
\td X_{k+1} =  \frac{s(t_k) + \dot s(t_{k})(t_{k+1}-t_k)}{s(t_{k+1})}\td X_k - \frac{s(t_k)}{s(t_{k+1})}(t_{k+1}-t_k) \dot \sigma(t_k) \sigma(t_k) \nabla\log p\op\td X_k; \sigma(t_k)^2\cp.
\end{equation}
This cannot be directly interpreted as $\alpha$-denoising, but for $s(t) = 1$ constant, then the update becomes
$$
\td X_{k+1} = \td X_k +\alpha_k\op\sigma(t_k)^2- \sigma(t_{k+1})^2\cp \nabla\log p\op\td X_k; \sigma(t_k)^2\cp,
$$
with $\alpha_k = \frac{- (t_{k+1}-t_k)\dot\sigma(t_k)\sigma(t_k)}{\sigma(t_{k})^2 - \sigma(t_{k+1})^2}$.
Here, we have steps of $\alpha$-denoising, but the coefficient $\alpha_k$ does not only depend on the noise levels $(\sigma(t_k))_k$, but also on the specific choice of noise parametrization $t\mapsto \sigma(t)$. Indeed, taking $\sigma(t) = \sqrt{t}$ (which corresponds to the model introduced in Section~\ref{sct:wasserstein}) gives
$$
\alpha _k = \frac{-(\sigma(t_{k+1})^2 - \sigma(t_k)^2)\frac{1}{2\sigma(t_k)}\sigma(t_k)}{\sigma(t_k)^2 - \sigma(t_{k+1})^2} = \frac{1}{2},
$$
while taking $\sigma(t) = t$ \citep[as][]{karrasElucidatingDesignSpace2022}, leads to
$$
\alpha _k = \frac{-(\sigma(t_{k+1}) - \sigma(t_k))\sigma(t_k)}{\sigma(t_k)^2 - \sigma(t_{k+1})^2} = \frac{\sigma(t_k)}{\sigma(t_{k+1})+\sigma(t_k)},
$$
which is exactly the same as the DDIM update.

For the parametrization $s(t) = 1-t$ and $\sigma(t) = \frac{t}{1-t}$ used in flow matching \citep{liuFlowStraightFast2023,lipmanFlowMatchingGenerative2023,albergoStochasticInterpolantsUnifying2023}, the update (\ref{eq:general_update_multisteps}) can also be interpreted as $\alpha_k$-denoising. Indeed, we have
$$
 \frac{s(t_k) + \dot s(t_{k})(t_{k+1}-t_k)}{s(t_{k+1})} = \frac{1-t_k - (t_{k+1}-t_k)}{1-t_{k+1}} = 1,
$$
and with the identities $t = \frac{\sigma(t)}{1+\sigma(t)}$, $s(t) = \frac{1}{1+\sigma(t)}$, and $\dot \sigma(t) = \frac{1}{(1-t)^2} = (1 + \sigma(t))^2$, we get that
\begin{align*}
\alpha_k &= \frac{s(t_k)}{s(t_{k+1})}\frac{- (t_{k+1}-t_k)\dot\sigma(t_k)\sigma(t_k)}{\sigma(t_{k})^2 - \sigma(t_{k+1})^2} \\ 
&= \frac{1+\sigma(t_{k+1})}{1+\sigma(t_{k})} \frac{\op \frac{\sigma(t_k)}{1+\sigma(t_k)} -  \frac{\sigma(t_{k+1})}{1+\sigma(t_{k+1})}\cp (1 + \sigma(t_k))^2 \sigma(t_k)}{(\sigma(t_{k}) + \sigma(t_{k+1}))(\sigma(t_{k}) - \sigma(t_{k+1}))} \\
&= \frac{\sigma(t_k)}{\sigma(t_{k+1})+\sigma(t_k)},
\end{align*}
which is exactly the same as the DDIM update.

We can interpret these different coefficients $\alpha_k$ in the light of our theoretical results.
At the beginning of the reverse process, one is going from a noisy distribution to a sightly less noisy one.
As noising regularizes the density, our finding shows that half-denoising is better. For the Euler discretization with noise schedule $\sigma(t) = \sqrt{t}$, we do have $\alpha_k = \frac{1}{2}$. For DDIM---and equivalently the Euler discretization with parametrization $s(t)=1,\, \sigma(t) = t$ or $s(t)=1-t,\, \sigma(t) = t/(1-t)$---if we assume that $\sigma(t_{k+1}) = \sigma(t_k) - \Delta\sigma$ with $ \Delta\sigma \ll \sigma(t_k)$, then we have $\alpha_k \approx \frac{1}{2}$.

On the contrary, at the end of the reverse process, one is going from a noisy distribution to the target distribution, therefore the choice of the coefficient $\alpha$ should depend on what we know about the target density.
For the Euler discretization with noise schedule $\sigma(t) = \sqrt{t}$, as the coefficients $\alpha_k$ are constant equal to $\frac{1}{2}$, our theoretical results suggest that this choice is best for a regular target density.
For DDIM, or Euler with $s(t)=1,\, \sigma(t) = t$ or $s(t)=1-t,\, \sigma(t) = t/(1-t)$, then $\alpha_k$ gradually shifts to 
$$\alpha_{N-1} =\frac{\sigma(t_{N-1})}{\sigma(t_{N-1})+\sigma(t_N)} = \frac{\sigma(t_{N-1})}{\sigma(t_{N-1})+0} = 1,$$
which corresponds to full-denoising.%
\footnote{Note that this is coherent with the fact that DDIM is exact if the data distribution is a Dirac \citep{nakkiranStepbystepDiffusionElementary2024}, a  distribution for which full-denoising is also exact.}
Our results suggest that this choice is best suited for a singular density, for example under the \emph{manifold hypothesis}.

When interpreting each step as $\alpha$-denoising, it appears that these design choices lead to different algorithms that may be more or less suited to certain target distributions.
We believe that it would be interesting in future works to study whether it is possible to choose directly different time schedules for the coefficient $\alpha$. More generally, as the optimal $\alpha$ depends on the properties of the data distribution, it would be interesting to see if it could be fine-tuned in a data-dependent way.

\section{Conclusion}
We have shown that half-denoising is better than full-denoising for regular enough densities, while full-denoising is better for singular densities such as mixtures of Dirac measures or Gaussian with small variance compare to the additional noise.
Moreover, the performance of the denoisers can be further accessed with additional assumptions on the data distribution, that occur naturally in real-world data, for example with images under the \emph{manifold hypothesis}.

When the variable is supported on a lower-dimensional subspace, we have shown that there is a trade-off between full-denoising which reduces the Wasserstein distance by ensuring that the output belongs to the subspace, and half-denoising that reduces the Wasserstein distance on the lower-dimensional subspace assuming a regular enough density. In the case where the subspace is of small enough dimension compared to the full space, full-denoising alleviates the curse of dimensionality as the Wasserstein distance only depends on the distance between distributions on the lower-dimensional subspace. Moreover, we have shown that the denoising performance for a mixture of distributions with disjoint compact supports behaves as if we were denoising each variable independently, plus an exponentially decreasing term.
This led to a case we called \emph{linear manifold hypothesis}, where the data distribution is supported on disjoint compact sets, each of these belonging to a (different) linear subspace of low dimension, and for which full-denoising can alleviate the curse of dimensionality even if the support of the distribution itself spans the whole space, as it adapts to the local linear structure of the distribution.

For algorithms using one step denoising \citep{saremiNeuralEmpiricalBayes2019,saremiChainLogConcaveMarkov2023,hyvarinenNoisecorrectedLangevinAlgorithm2025}, our results provide guidance on choosing between half-denoising (regular density) and full-denoising (singular target density, manifold hypothesis).
Moreover, for multiple steps denoising models, we have shown that each step can be seen as $\alpha$-denoising, with different $\alpha$'s depending on design choices.
Our theoretical results therefore offer new insights into design choices based on assumptions about the data distribution.

There are several avenues to explore to extend this work.
We could try to extend our results for the \emph{linear manifold} to a more general low-dimensional manifold, drawing inspiration from \cite{azangulovConvergenceDiffusionModels2025} to get a result in their setting, while improving the dependency in $\sigma$ and relaxing some assumptions. 
It would also be interesting to explore further around diffusion models, for which multiple denoising steps are repeated, and compare to existing theoretical result on the performance of the models \citep[see, e.g.,][]{bortoliConvergenceDenoisingDiffusion2022,chenSamplingEasyLearning2023,confortiKLConvergenceGuarantees2025,gentiloni-silveriLogConcavityScoreRegularity2025}.
We believe that the idea of tuning the coefficient $\alpha$ of the denoising steps depending on the target distribution could be further pursued.
We also hope that the techniques developed here could lead to a better understanding of the Wasserstein error of diffusion models, both for deterministic sampling (discretization of the diffusion ODE (\ref{eq:ODE})) and stochastic sampling (discretization of a reverse time SDE).
To do so, one should also take into account the error at initialization (approximating the noisy distribution by a Gaussian), the impact of the (possibly varying) step size of the noise schedule \citep{strasmanAnalysisNoiseSchedule2024} and the error in learning the score.
For the latter, it would be interesting to include results on the ability of neural networks to learn the score for distributions supported on a lower-dimensional manifold, as done by \citet{tangAdaptivityDiffusionModels2024} and by \citet{azangulovConvergenceDiffusionModels2025}.
We tried to answer some of these questions related to the Wasserstein convergence guarantees of the multiple step diffusion models in a subsequent paper \citep{beylerConvergenceDeterministicStochastic2025}.

\acks{%
We thank the Action Editor and the referees for their valuable feedback. We thank Saeed Saremi and  Dario Shariatian for insightful discussions related to this work.
This work has received support from the French government, managed by the National Research Agency, under the France 2030 program with the reference ``PR[AI]RIE-PSAI'' (ANR-23-IACL-0008).
}

\newpage
\appendix
\appendixpage

\section{Fokker-Planck Equation and the Diffusion ODE}
\label{sct:fokker}

In this section, we give a short proof of the Fokker-Planck equation, and use it to define the ODE (\ref{eq:ODE}) presented in section \ref{sct:wasserstein}. There are many references about Fokker-Planck equations \citep[see, e.g.,][]{riskenFokkerPlanckEquationMethods1996,bogachevFokkerPlanckKolmogorov2015}, but here we try to give a proof as simple and self-contain as possible. For general results on SDEs, we refer the reader to \citet{legallBrownianMotionMartingales2016}. We will always assume, but note explicitly write, that we have a probability space, a filtration and a Brownian motion such that all randoms variables are well defined.

We study the process $X_t = X + B_t$, in particular, $X_t = X + t Z$ with $X\bot Z$, $Z \sim \N(0,I)$. From that we deduce that $X_t$ admit a density $p_t$ with respect to the Lebesgue measure that verifies, for $t>0$,
$$
p_t(x) = \int \frac{\exp\op -\Vert x-u\Vert^2/2t\cp}{(2\pi t)^{d/2}} d\mu_X(u) > 0,
$$
and that $(t,x) \mapsto p_t(x)$ is $\C^\infty$.

The evolution of the marginal $p_t$ is dictated by a partial differential equation, the Fokker-Planck equation.

\begin{proposition}[Fokker-Planck equation]
\label{prop:fokker_planck}
Let $f \in \C^1(\R\times\R^d,\R^d)$, $g \in \C^0(\R,\R)$, and $p_t$ the density of a process $X_t$ that verifies the following SDE
\begin{equation}
\label{eq:general_SDE}
dX_t = f(t,X_t)dt + g(t)dB_t,
\end{equation}
and assume that $(t,x) \mapsto p_t(x)$ is $\C^2$ on $\R_{+}^*\times\R^d$.
Then for all $t>0$, we have
$$
\partial_t p_t = -\nabla \cdot (fp_t) + \frac{1}{2}g^2\Delta p_t.
$$
\end{proposition}

\begin{proof}
Let $\varphi \in \C^\infty(\R^d,\R)$ with compact support. Then for $t>0$,
\begin{align*}
\int\varphi(x)\partial_tp_t(x) dx 
&= \int\varphi(x) \lim_{h\rightarrow0}\frac{p_{t+h}(x)-p_t(x)}{h}dx\\
&= \lim_{h\rightarrow0}\frac{1}{h} \op \int\varphi(x)p_{t+h}(x)dx  - \int\varphi(x)p_{t}(x)dx\cp \\
&=\lim_{h\rightarrow0}\frac{1}{h} \E[\varphi(X_{t+h}) - \varphi(X_t)].
\end{align*}
By Itō's formula,
\begin{align*}
\varphi(X_{t+h}) - \varphi(X_t) 
&= \int_t^{t+h} \nabla\varphi(X_s)\cdot dX_s + \frac{1}{2}\int_t^{t+h} \Delta \varphi(X_s)d\langle X,X\rangle_s \\
&= \int_t^{t+h} \nabla\varphi(X_s)\cdot f(s,X_s)ds +\int_t^{t+h}g(s)\nabla\varphi(X_s)\cdot dB_s \\
&\quad + \frac{1}{2}\int_t^{t+h} \Delta \varphi(X_s)g(s)^2ds.
\end{align*}
Taking the expectation, we get that,
$$
\E[\varphi(X_{t+h}) - \varphi(X_t)]= \int_t^{t+h}\E[\nabla\varphi(X_s)\cdot f(s,X_s)] ds +0 + \frac{1}{2}\int_t^{t+h} E[\Delta \varphi(X_s)]g(s)^2ds.
$$
As $\varphi$ has compact support, integration by part gives:
$$
\E[\nabla\varphi(X_s)\cdot f(s,X_s)] = \int \nabla\varphi(x)\cdot f(s,x)p_s(x)dx = - \int \varphi(x)\nabla \cdot (f(s,x)p_s(x))dx,
$$
and
$$
E[\Delta \varphi(X_s)] = \int \Delta \varphi(x)p_s(x)dx =  \int \varphi(x)\Delta  p_s(x)dx.
$$
It follows that,
\begin{align*}
\int&\varphi(x)\partial_tp_t(x) dx \\
&=\lim_{h\rightarrow0}\frac{1}{h} \op-\int_t^{t+h}\int\varphi(x)\nabla\cdot (f(s,x)p_s(x))dxds + \frac{1}{2}\int_t^{t+h}\int \varphi(x)g(s)^2\Delta  p_s(x)dxds\cp \\
&=-\int\varphi(x)\nabla\cdot (f(t,x)p_t(x))dx + \frac{1}{2}\int \varphi(x)g(t)^2\Delta  p_t(x)dx,
\end{align*}
as both integrands are continuous with respect to time, and we can rewrite,
$$
\int\varphi(x)\op\partial_tp_t(x) 
+\nabla\cdot (f(t,x)p_t(x))  - \frac{1}{2}g(t)^2\Delta  p_t(x)\cp dx = 0,
$$
for all $\varphi \in \C^\infty(\R^d,R)$ with compact support. It follows that for all $t>0$, almost surely in~$x$,
$$
\partial_tp_t(x) 
 = -\nabla\cdot (f(t,x)p_t(x)) + \frac{1}{2}g(t)^2\Delta  p_t(x),
$$
which gives the desired result as these quantities are continuous in $x$.
\end{proof}

For the process $X_t = X + B_t$, we have $dX_t = dB_t$, which is (\ref{eq:general_SDE}) with $f=0$ and $g=1$. The Fokker-Planck equation is therefore 
$$
\partial_t p_t =  \frac{1}{2}\Delta p_t.
$$
As $\Delta p_t = \nabla\cdot\nabla p_t = \nabla\cdot\frac{p_t}{p_t}\nabla p_t = \nabla\cdot p_t\nabla\log p_t$, the equation can be rewritten as 
$$
\partial_t p_t =  - \nabla\cdot\op-\frac{1}{2}\nabla\log p_t\cp p_t,
$$
which correspond to the Fokker-Planck equation with drift term $f(t,x) = -\frac{1}{2}\nabla\log p_t(x)$. This allows use to define the diffusion ODE, used in section \ref{sct:wasserstein}.

\begin{proposition}
\label{prop:diffusion_ODE}
Let $t^* >0$. We can define a process $(x_t)_{t\geq0}$ by
\begin{equation}
\tag{\ref{eq:ODE}}
\left\{\begin{array}{rll}
    \frac{d x_t}{dt} &= -\frac{1}{2} \nabla \log p_t(x_t) &\text{for } t > 0\\
    x_{t^*} &= X_{t^*}.
\end{array}\right.
\end{equation}
This process has the same marginals as $X_t$: $\forall t \in [0,+\infty[, x_t \sim X_t$, and verifies, for all $t,s \geq 0$,
$$
x_t - x_s = -\frac{1}{2}\int_s^t \nabla \log p_u(x_u) du. 
$$
\end{proposition}
\begin{proof}
We know that for $t>0$, $(t,x) \mapsto p_t(x)$ is $\C^\infty$. In particular, the ODE defined by (\ref{eq:ODE}) can be solved for all time $t>0$, and we have for $s,t >0$:
\begin{equation}
\label{eq:int_ODE}
x_t - x_s = -\frac{1}{2}\int_s^t \nabla \log p_u(x_u) du.
\end{equation}

For now, we will write $\tilde{p}_t$ the density of the marginal of $x_t$. As $x_t$ is defined as the solution of an ODE, we introduce the resolvent $R(s,t,x)$ that gives the solution at time $t$ of the ODE $$y' = -\frac{1}{2} \nabla \log p_t(y),$$ with initial condition $x$ at time $s$. $R$ is invertible and verifies $R(s,t,x)^{-1} = R(t,s,x)$. Moreover, as $(t,x) \mapsto \nabla \log p_t(x)$ is $C^\infty$, $R$ is also $\C^\infty$ \citep[see, e.g.][Théorème 7.21]{paulinTopologieAnalyseCalcul2009}.

By construction, $x_t = R(t^*,t,X_{t^*})$, leading to:
$$
\tilde{p}_t(x) = |\det \nabla R(t,t^*,x)| p_{t^*}(R(t,t^*,x)),
$$
in particular, $(t,x)\mapsto\tilde{p}_t(x)$ is $\C^\infty$ on $\R_{+}^*\times\R^d$. We can apply proposition \ref{prop:fokker_planck}, to get that 
$$
\partial_t \tilde{p}_t = \nabla\cdot\op-\frac{1}{2} \nabla \log p_t\cp = -\frac{1}{2}\Delta \log p_t = \partial_tp_t,
$$
and as $\tilde{p}_{t^*} = p_{t^*}$, it leads to $\tilde{p}_{t} = p_{t}$ for all $t>0$.

We finally need to verify that we can extend (\ref{eq:int_ODE}) up to time $s = 0$, i.e., that the trajectories can be integrated up to time $t=0$. For $t>0$, Tweedie's formula (\ref{eq:tweedie}) gives:
$$
\nabla\log p_t(x) = \frac{1}{t}(x - \E[X|X_t = x]).
$$
In particular, as $x_t$ and $X_t$ have the same distribution,
$$
\E[\Vert \nabla\log p_t(x_t)\Vert] = \frac{1}{t}\E[\Vert X_t - \E[X|X_t]\Vert] = \frac{1}{t}\E[\Vert\E[X_t - X|X_t]\Vert].
$$
Jensen's inequality gives 
\begin{align*}
\E[\Vert\E[X_t - X|X_t]\Vert] &\leq \E[\E[\Vert X_t - X\Vert|X_t]] = \E[\Vert X_t - X\Vert] \\
&\leq  \sqrt{\E[\Vert X - X_t\Vert^2]} = \sqrt{\E[\Vert B_t\Vert^2]} = \sqrt{td},
\end{align*}
therefore,
$$
\E[\Vert \nabla\log p_t(x_t)\Vert] \leq \sqrt{\frac{d}{t}},
$$
which is integrable near $0$. We get that
$$
\E\ob\int_0^t \Vert \nabla \log p_u(x_u)\Vert du \cb < \infty,
$$
hence 
$$
\int_0^t \Vert \nabla \log p_u(x_u)\Vert du < \infty,
$$
almost everywhere, proving that we can extend (\ref{eq:int_ODE}) up to time $s = 0$ (almost everywhere).
\end{proof}

\section{Proofs}
\label{sct:proofs}
\subsection{Proofs of Proposition \ref{prop:MMD} and Corollary \ref{cor:MMD}}
We will first state the following lemma:
\begin{lemma}
\label{lemma:expectation_gradient_noise}
Let $X$ be a random variable with density $p_X$ such that  $\E[\Vert \nabla \log p_X(X)\Vert^p]\leq C$ for some constant $C < \infty$ and  $p\geq 1$. Then for all $\sigma >0$, defining $Y = X + \ep$  with $X \bot \ep$ and $\ep \sim \N(0, \sigma^2 I)$, we have,
$$\E[\Vert \nabla \log p_Y(Y)\Vert^p]\leq C.$$
\end{lemma}
\begin{proof}
From (B.2) of \cite{saremiChainLogConcaveMarkov2023}, $\nabla \log p_Y(y)= \E [\nabla \log p_X(X) | Y=y]$ hence:
\begin{align*}
\E[\Vert \nabla \log p_Y(Y)\Vert^p] &= \E[\Vert \E [\nabla \log p_X(X) | Y]\Vert^p]\\
\text{(Jensen's inequality)} &\leq \E[\E[\Vert\nabla \log p_X(X)\Vert^p|Y]]\\
& = \E[\Vert \nabla \log p_X(X)\Vert^p]\leq C.
\end{align*}
\end{proof}

We can now prove Proposition \ref{prop:MMD}. 
\begin{proof}(Proposition \ref{prop:MMD})

We follow \citet{hyvarinenNoisecorrectedLangevinAlgorithm2025} and keep track of the different constants. 

\vspace{.2cm}
\textbf{For $\alpha = \frac{1}{2}$:} Let $\hat{X} = \varphi_{1/2}(Y)$. As $Y = X + \ep$ with $X \bot \ep$ and $\ep \sim \N(0, \sigma^2 I)$, we have $\hat{p}_Y(\xi) = \hat{p}_X(\xi) e^{-\frac{\sigma^2}{2} \Vert \xi\Vert^2}$.

The proof uses the equality:
$$\E[e^{i \xi\cdot Y} i\frac{\sigma^2}{2}\xi\cdot\nabla \log p_Y(Y)] =\frac{\sigma^2}{2} \Vert\xi\Vert^2 \hat{p}_Y(\xi),$$
that comes from an integration by parts ((19) of \cite{hyvarinenNoisecorrectedLangevinAlgorithm2025}). Then the idea is to write $\hat{p}_{\hat{X}}(\xi) = \E[e^{i \xi\cdot \hat{X}}] = \E[e^{i \xi\cdot Y}e^{i \frac{\sigma^2}{2}\xi\cdot\nabla \log p_Y(Y)}] = \E[e^{i \xi\cdot Y}(1 + i \frac{\sigma^2}{2}\xi\cdot\nabla \log p_Y(Y) + O(\sigma^4)] = \E[e^{i \xi\cdot Y}] + \E[e^{i \xi\cdot Y} i\frac{\sigma^2}{2}\xi\cdot\nabla \log p_Y(Y)] + O(\sigma^4) = \hat{p}_Y(\xi)\op 1+\frac{\sigma^2}{2} \Vert\xi\Vert^2\cp + O(\sigma^4) =  \hat{p}_Y(\xi) e^{\frac{\sigma^2}{2} \Vert \xi\Vert^2} + O(\sigma^4) = \hat{p}_X(\xi) + O(\sigma^4)$.

To make it quantitative, we start by noticing that for any $z \in \mathbb{C}$, $d\in \mathbb{N}$, 
$$\left| e^{z} - \sum_{k=0}^{d} \frac{z^k}{k!}\right| \leq \frac{|z|^{d+1} \max(1,e^{\Re(z)})}{(d+1)!}.$$
Then
\begin{align*}
|\hat{p}_{\hat{X}}(\xi) - (1 + 1/2 \sigma^2 \Vert\xi\Vert^2)\hat{p}_Y(\xi)| &\leq \E[|e^{i \xi\cdot Y}|\cdot|e^{i \frac{1}{2}\sigma^2\xi\cdot\nabla \log p_Y(Y)} - (1 + i\frac{1}{2}\sigma^2\xi\cdot\nabla \log p_Y(Y))|] \\
&\leq \E[\frac{1}{2}(\frac{1}{2}\sigma^2\xi\cdot\nabla \log p_Y(Y))^2] \\
&\leq \frac{\sigma^4\Vert\xi\Vert^2}{8}\E[\Vert\nabla \log p_Y(Y)\Vert^2] \\
&\leq \frac{\sigma^4 \Vert\xi\Vert^2 C}{8} \quad\text{(Lemma \ref{lemma:expectation_gradient_noise})},
\end{align*}
and,
\begin{align*}
|\hat{p}_X(\xi) - (1 + 1/2 \sigma^2 \Vert\xi\Vert^2)\hat{p}_Y(\xi)| &= |\hat{p}_X(\xi) ||1 - (1 + 1/2 \sigma^2 \Vert\xi\Vert^2)e^{-\frac{1}{2}\sigma^2 \Vert \xi\Vert^2}|\\
&\leq |1 - (1 + 1/2 \sigma^2 \Vert\xi\Vert^2)e^{-\frac{1}{2}\sigma^2 \Vert \xi\Vert^2}| \\
&= e^{-\frac{1}{2}\sigma^2 \Vert \xi\Vert^2} |e^{\frac{1}{2}\sigma^2 \Vert \xi\Vert^2}  - (1 + \frac{1}{2} \sigma^2 \Vert\xi\Vert^2)|\\
&\leq e^{-\frac{1}{2}\sigma^2 \Vert \xi\Vert^2} e^{\frac{1}{2}\sigma^2 \Vert \xi\Vert^2} \frac{1}{2} (\frac{1}{2} \sigma^2 \Vert\xi\Vert^2)^2 \\
&= \frac{\sigma^4 \Vert \xi\Vert^4}{8}.
\end{align*}
Finally we get the result by combining the two inequalities.

\vspace{.2cm}
\textbf{For $\alpha \neq \frac{1}{2}$:} Let $\hat{X} = \varphi_{\alpha}(Y)$. We can still write $\hat{p}_{\hat{X}}(\xi) = \E[e^{i \xi\cdot \hat{X}}] = \E[e^{i \xi\cdot Y}e^{i \alpha\sigma^2\xi\cdot\nabla \log p_Y(Y)}] = \E[e^{i \xi\cdot Y}(1 + i \alpha\sigma^2 \xi\cdot\nabla \log p_Y(Y) + O(\sigma^4)] = \E[e^{i \xi\cdot Y}] + \E[e^{i \xi\cdot Y} i\alpha\sigma^2\xi\cdot\nabla \log p_Y(Y)] + O(\sigma^4) = \hat{p}_Y(\xi)\op 1+\alpha \sigma^2 \Vert\xi\Vert^2\cp + O(\sigma^4) =  \hat{p}_Y(\xi) e^{\alpha \sigma^2 \Vert \xi\Vert^2} + O(\sigma^4) = \hat{p}_X(\xi) e^{\op\alpha -\frac{1}{2}\cp \sigma^2 \Vert \xi\Vert^2}+ O(\sigma^4)$. But here the term $\op\alpha -\frac{1}{2}\cp$ does not cancel, hence we do not have $\hat{p}_{\hat{X}}(\xi) = \hat{p}_X(\xi) + O(\sigma^4) $. However, we can still write that $e^{\op\alpha -\frac{1}{2}\cp \sigma^2 \Vert \xi\Vert^2} = 1  + O(\sigma^2)$ to have the equality, but to a smaller order in $\sigma$. Quantitatively: 
\begin{align*}
|\hat{p}_{\hat{X}}(\xi) - \hat{p}_Y(\xi)| &\leq \E[|e^{i \xi\cdot Y}|\cdot|e^{i \alpha\sigma^2\xi\cdot\nabla \log p_Y(Y)} - 1|] \\
&\leq \E[|\alpha\sigma^2\xi\cdot\nabla \log p_Y(Y)|] \\
&\leq \alpha\sigma^2\Vert\xi\Vert\E[\Vert\nabla \log p_Y(Y)\Vert] \\
&\leq \alpha\sigma^2\Vert\xi\Vert\sqrt{\E[\Vert\nabla \log p_Y(Y)\Vert^2]}\\
&\leq \alpha\sigma^2 \Vert\xi\Vert \sqrt{C}\quad \text{(Lemma \ref{lemma:expectation_gradient_noise})},
\end{align*}
and,
\begin{align*}
|\hat{p}_X(\xi) - \hat{p}_Y(\xi)| &= |\hat{p}_X(\xi) |\cdot|1 - 1 e^{-\frac{1}{2}\sigma^2 \Vert \xi\Vert^2}|\\
&\leq |1 - e^{-\frac{1}{2}\sigma^2 \Vert \xi\Vert^2}| \\
&= e^{-\frac{1}{2}\sigma^2 \Vert \xi\Vert^2} |e^{\frac{1}{2}\sigma^2 \Vert \xi\Vert^2}  - 1|\\
&\leq e^{-\frac{1}{2}\sigma^2 \Vert \xi\Vert^2} e^{\frac{1}{2}\sigma^2 \Vert \xi\Vert^2}(\frac{1}{2} \sigma^2 \Vert\xi\Vert^2) \\
&= \frac{\sigma^2 \Vert \xi\Vert^2}{2},
\end{align*}
which gives the result by combining the two inequalities.
\end{proof}

\begin{proof}(Corollary \ref{cor:MMD})

\textbf{For $\alpha = \frac{1}{2}$:}
\begin{align*}
\MMD_k(\L(X),\L(\varphi_{1/2}(Y))) &= \op \int |\hat{p}_X(\xi) - \hat{p}_{\varphi_{1/2}(Y)}(\xi)|^2 d\Lambda(\xi) \cp^{1/2} \\
&\leq \op \int \op\frac{\sigma^4 (C\Vert\xi\Vert^2 + \Vert\xi\Vert^4)}{8}\cp^2 d\Lambda(\xi) \cp^{1/2} \\
&\leq \frac{\sigma^4}{8} \op 2C^2 \int \Vert\xi\Vert^4d\Lambda(\xi) + 2 \int \Vert\xi\Vert^8 d\Lambda(\xi) \cp^{1/2} \\
&= \sigma^4 \frac{\sqrt{C^2 C_4 + C_8}}{4\sqrt{2}}.
\end{align*}

\vspace{.2cm}
\textbf{For $\alpha \neq \frac{1}{2}$:} 
\begin{align*}
\MMD_k(\L(X),\L(\varphi_{\alpha}(Y))) &= \op \int |\hat{p}_X(\xi) - \hat{p}_{\varphi_{\alpha}(Y)}(\xi)|^2 d\Lambda(\xi) \cp^{1/2} \\
&\leq \op \int \op\frac{\sigma^2 (2\alpha\sqrt{C}\Vert\xi\Vert + \Vert\xi\Vert^2)}{2}\cp^2 d\Lambda(\xi) \cp^{1/2} \\
&\leq \frac{\sigma^2}{2} \op 8\alpha^2 C \int \Vert\xi\Vert^2d\Lambda(\xi) + 2 \int \Vert\xi\Vert^4 d\Lambda(\xi) \cp^{1/2} \\
&= \sigma^2 \frac{\sqrt{4 \alpha^2 C C_2 + C_4}}{\sqrt{2}}.
\end{align*}
\end{proof}

\subsection{Proofs of Proposition \ref{prop:W2}, Lemma \ref{lemma:derivative_bound} and Proposition \ref{prop:W2half}}

\begin{proof}(Proposition \ref{prop:W2})

We have 
$$x_0 =  x_t + \frac{1}{2}\int_0^t \nabla \log p_s(x_s) ds,$$
and 
$$\hat{x}_0 = x_t + \alpha t \nabla \log p_t(x_t).$$
Therefore,
$$
x_0 - \hat{x}_0 = \frac{1}{2}\int_0^t \nabla \log p_s(x_s) ds - \alpha t \nabla \log p_t(x_t).
$$
Using Jensen's inequality two times,
\begin{align*}
\Vert x_0 - \hat{x}_0\Vert^2
&\leq 2\op\left\Vert\frac{1}{2}\int_0^t \nabla \log p_s(x_s) ds\right\Vert^2  + \Vert\alpha t \nabla \log p_t(x_t)\Vert^2\cp\\
&\leq 2\op\frac{t}{4}\int_0^t \Vert\nabla \log p_s(x_s)\Vert^2ds + \alpha^2 t^2 \Vert\nabla \log p_t(x_t)\Vert^2\cp,  
\end{align*}
and it follows that,
$$
W_2^2(\L(X),\L(\varphi_\alpha(X_t))) \leq  \E[\Vert x_0 - x_t\Vert^2] \leq 2\op\frac{t}{4}\int_0^t \E[\Vert\nabla \log p_s(x_s)\Vert^2]ds + \alpha^2 t^2 \E[\Vert\nabla \log p_t(x_t)\Vert^2]\cp.
$$

As stated in Lemma \ref{lemma:expectation_gradient_noise}, assuming $\E[\Vert\nabla \log p_X(X)\Vert^2] \leq C$ leads to $\E[\Vert\nabla \log p_s(x_s)\Vert^2] \leq C$ for all $s\geq 0$, so finally:
$$
W_2^2(\L(X),\L(\varphi_\alpha(X_t))) \leq \frac{(1+4\alpha^2)C}{2}t^2,
$$
which gives for $t = \sigma^2$:
$$
W_2(\L(X),\L(\varphi_\alpha(Y))) \leq \sqrt{\frac{(1+4\alpha^2)C}{2}} \sigma^2.
$$
\end{proof}

\begin{proof}(Lemma \ref{lemma:derivative_bound})
Using the chain rule and equation (\ref{eq:ODE}), we have
\begin{align*}
\frac{d}{dt}\nabla \log p_t(x_t) &= \nabla^2\log p_t(x_t)\cdot\frac{d}{dt}x_t + [\partial_t\nabla \log p_t](x_t) \\
&= -\frac{1}{2}\nabla^2\log p_t(x_t)\cdot\nabla\log p_t(x_t) + [\nabla \partial_t\log p_t](x_t).
\end{align*}
Moreover, the Fokker-Planck equation for $p_t$ is
$$
\partial_t p_t = \frac{1}{2} \Delta  p_t = \frac{1}{2} \nabla\cdot\nabla p_t =\frac{1}{2} \nabla\cdot\op p_t \nabla\log p_t\cp = \frac{1}{2} \nabla p_t \cdot \nabla\log p_t + \frac{1}{2} p_t \Delta  \log p_t,
$$
hence,
$$
\partial_t \log p_t = \frac{1}{2} \Vert \nabla \log p_t\Vert^2 + \frac{1}{2}\Delta  \log p_t.
$$
Taking the gradient in $x$ gives
$$
\nabla\partial_t \log p_t = \nabla^2 \log p_t \cdot \nabla \log p_t + \frac{1}{2}\nabla\Delta  \log p_t.
$$
This finally leads to 
$$
\frac{d}{dt}\nabla \log p_t(x_t) = \frac{1}{2}\nabla^2\log p_t(x_t)\cdot\nabla\log p_t(x_t) + \frac{1}{2}\nabla\Delta  \log p_t (x_t).
$$

We can express $\nabla\log p_t$, $\nabla^2\log p_t$ and $\nabla\Delta  \log p_t$ in term of $\nabla\log p_X$, $\nabla^2\log p_X$ and $\nabla\Delta  \log p_X$. With $Y = X + \ep$ with $X \bot \ep$ and $\ep \sim \N(0, t I)$, (B.2) and (B.4) of \citet{saremiChainLogConcaveMarkov2023} give
$$
\nabla\log p_t (x) = \E[\nabla\log p_X(X)|Y = x],
$$
and 
\begin{align*}
\nabla^2\log p_t(x) 
&= \E[\nabla^2\log p_X(X)|Y=x]\\
&\quad + \E[\nabla\log p_X(X) \nabla\log p_X(X)^\top|Y=x]\\
&\quad - \E[\nabla \log p_X(X)|Y=x]\E[\nabla\log p_X(X)|Y=x]^\top \\
&= \E[\nabla^2\log p_X(X)|Y=x] + \text{cov}(\nabla\log p_X(X)|Y=x).
\end{align*}
Similarly, we can compute,
\begin{align*}
\nabla\Delta  \log p_t(x)
&= \E[\nabla\Delta  \log p_X(X)|Y=x]\\
&\quad+ \E[\Delta  \log p_X(X) \nabla\log p_X(X) |Y=x]\\
&\quad- \E[\Delta  \log p_X(X)|Y=x]\E[\nabla\log p_X(X)|Y=x]\\
&\quad+ 2\E[\nabla^2\log p_X(X)\cdot\nabla\log p_X(X)|Y=x]\\
&\quad+ \E[\Vert \nabla\log p_X(X)\Vert^2 \nabla\log p_X(X)|Y=x]\\
&\quad- \E[\Vert \nabla\log p_X(X)\Vert^2|Y=x]\E[\nabla\log p_X(X)|Y=x]\\
&\quad- 2\E[\nabla^2\log p_X(X)|Y=x] \cdot \E[\nabla\log p_X(X)|Y = x] \\
&\quad- 2\E[\nabla\log p_X(X) \nabla\log p_X(X)^\top|Y=x] \cdot \E[\nabla\log p_X(X)|Y = x]\\
&\quad+ 2\Vert\E[\nabla \log p_X(X)|Y=x]\Vert^2 \E[\nabla \log p_X(X)|Y=x].
\end{align*}

Combining the expressions above, we get:
\begin{align}
\frac{d}{dt}\nabla \log p_t(x_t)
&= \frac{1}{2}\E[\nabla\Delta  \log p_X(X)|Y=x_t] \label{term:1}\\
&\quad+ \frac{1}{2}\E[\Delta  \log p_X(X) \nabla\log p_X(X) |Y=x_t] \label{term:2}\\
&\quad- \frac{1}{2}\E[\Delta  \log p_X(X)|Y=x_t]\E[\nabla\log p_X(X)|Y=x_t] \label{term:3}\\
&\quad+ \E[\nabla^2\log p_X(X)\cdot\nabla\log p_X(X)|Y=x_t] \label{term:4}\\
&\quad+ \frac{1}{2}\E[\Vert \nabla\log p_X(X)\Vert^2 \nabla\log p_X(X)|Y=x_t] \label{term:5}\\
&\quad- \frac{1}{2}\E[\Vert \nabla\log p_X(X)\Vert^2|Y=x_t]\E[\nabla\log p_X(X)|Y=x_t] \label{term:6}\\
&\quad- \frac{1}{2}\E[\nabla^2\log p_X(X)|Y=x_t] \cdot \E[\nabla\log p_X(X)|Y = x_t] \label{term:7}\\
&\quad- \frac{1}{2}\E[\nabla\log p_X(X) \nabla\log p_X(X)^\top|Y=x_t] \cdot \E[\nabla\log p_X(X)|Y = x_t] \label{term:8}\\
&\quad+ \frac{1}{2}\Vert\E[\nabla \log p_X(X)|Y=x_t]\Vert^2 \E[\nabla \log p_X(X)|Y=x_t].\label{term:9}
\end{align}

We want to control all these terms in $L_2$ norm. First note that for all $x \in \R^d$, $|\Delta  \log p_X(x)| = |\tr(\nabla^2 \log p_X(x))| \leq \sqrt{d} \Vert\nabla^2 \log p_X(x))\Vert_{\text{fro}} \leq d\Vert \nabla^2 \log p_X(x)\Vert_\text{op}$, where  $\Vert A\Vert_\text{fro}$ is defined for any matrix $A$ as $\Vert A\Vert_\text{fro} = \sqrt{\tr(AA^\top)}$ and verifies $\Vert A\Vert_\text{fro}\leq \sqrt{d}\Vert A\Vert_\text{op}$. In particular, $\E[|\Delta  \log p_X(X)|^3]\leq d^3C_2$. Also note that $x_t$ has the same distribution as $Y$. We can now control in expectation all the the terms in $\frac{d}{dt}\nabla \log p_t(x_t)$.

Term (\ref{term:1}):
\begin{align*} 
\E[\Vert\E[\nabla\Delta  \log p_X(X)|Y]\Vert^2] 
&\leq \E[\E[\Vert\nabla\Delta  \log p_X(X)\Vert^2|Y]] = \E[\Vert\nabla\Delta  \log p_X(X)\Vert^2]\\
&\quad\text{(Jensen's inequality on the conditional expectation)} \\
&= C_3.
\end{align*}

Term (\ref{term:2}):
\begin{align*} 
\E[\Vert\E[\Delta  \log p_X(X) \nabla\log p_X(X) |Y]\Vert^2] 
&\leq \E[|\Delta  \log p_X(X)|^2\Vert\nabla\log p_X(X)\Vert^2] \\
&\quad\text{(Jensen's inequality on the conditional expectation)} \\
&\leq \E[|\Delta  \log p_X(X)|^3]^{2/3} \E[\Vert\nabla\log p_X(X)\Vert^6]^{1/3}\\
&\quad\text{(Hölder's inequality)} \\
&\leq d^2 C_2^{2/3}C_1^{1/3}.
\end{align*}

Term (\ref{term:3}):
\begin{align*} 
\E[\Vert\E[\Delta  \log p_X(X)|Y]&\E[\nabla\log p_X(X)|Y]\Vert^2] \\
&= \E[|\E[\Delta  \log p_X(X)|Y]|^2 \Vert\E[\nabla\log p_X(X)|Y]\Vert^2] \\
&\leq \E[|\E[\Delta  \log p_X(X)|Y]|^3]^{2/3} \E[\Vert\E[\nabla\log p_X(X)|Y]\Vert^6]^{1/3} \\
&\quad\text{(Hölder's inequality)} \\
&\leq \E[|\Delta  \log p_X(X)|^3]^{2/3} \E[\Vert\nabla\log p_X(X)\Vert^6]^{1/3} \\
&\quad\text{(Jensen's inequality on the conditional expectation)} \\
&\leq d^2 C_2^{2/3}C_1^{1/3}.
\end{align*}

Term (\ref{term:4}):
 \begin{align*} 
 \E[\Vert\E[\nabla^2\log p_X(X)&\cdot\nabla\log p_X(X)|Y]\Vert^2]  \\
&\leq \E[\Vert\nabla^2\log p_X(X)\cdot\nabla\log p_X(X)\Vert^2]\\
&\quad\text{(Jensen's inequality on the conditional expectation)} \\
&\leq \E[\Vert\nabla^2\log p_X(X)\Vert_\text{op}^2 \Vert\nabla\log p_X(X)\Vert^2]\\
&\leq \E[\Vert\nabla^2\log p_X(X)\Vert_\text{op}^3]^{2/3} \E[\Vert\nabla\log p_X(X)\Vert^6]^{1/3}\\
&\quad\text{(Hölder's inequality)} \\
&= C_2^{2/3}C_1^{1/3}.
 \end{align*}
 
Term (\ref{term:5}):
\begin{align*} 
\E[\Vert\E[\Vert \nabla\log p_X(X)\Vert^2 &\nabla\log p_X(X)|Y]\Vert^2] \\
&\leq \E[\Vert \nabla\log p_X(X)\Vert^6] \\
&\quad\text{(Jensen's inequality on the conditional expectation)} \\
&= C_1.
\end{align*}

Term (\ref{term:6}):
\begin{align*} 
\E[\Vert\E[\Vert \nabla\log p_X(X)\Vert^2|Y]&\E[\nabla\log p_X(X)|Y]\Vert^2]\\
&=\E[(\E[\Vert \nabla\log p_X(X)\Vert^2|Y])^2 \Vert\E[\nabla\log p_X(X)|Y]\Vert^2]\\
&\leq \E[(\E[\Vert \nabla\log p_X(X)\Vert^2|Y])^2 \E[\Vert\nabla\log p_X(X)\Vert^2|Y]] \\
&\quad\text{(Jensen's inequality on the conditional expectation)} \\
&= \E[(\E[\Vert \nabla\log p_X(X)\Vert^2|Y])^3] \\
&\leq \E[\Vert \nabla\log p_X(X)\Vert^6] \\
&\quad\text{(Jensen's inequality on the conditional expectation)} \\
&= C_1.
\end{align*}

Term (\ref{term:7}):
\begin{align*} 
\E[\Vert\E[\nabla^2\log p_X(X)|Y]&\cdot\E[\nabla\log p_X(X)|Y]\Vert^2]\\
&\leq \E[\Vert\E[\nabla^2\log p_X(X)|Y]\Vert^2_\text{op}\Vert\E[\nabla\log p_X(X)|Y]\Vert^2] \\
&\leq \E[\Vert\E[\nabla^2\log p_X(X)|Y]\Vert^3_\text{op}]^{2/3} \E\Vert\E[\nabla\log p_X(X)|Y]\Vert^6]^{1/3} \\
&\quad\text{(Hölder's inequality)} \\
&\leq \E[\Vert\nabla^2\log p_X(X)\Vert^3_\text{op}]^{2/3} \E\Vert\nabla\log p_X(X)\Vert^6]^{1/3} \\
&\quad\text{(Jensen's inequality on the conditional expectation)} \\
&= C_2^{2/3}C_1^{1/3}.
\end{align*}

Term (\ref{term:8}):
\begin{align*} 
\E[\Vert\E[\nabla\log p_X(X) &\nabla\log p_X(X)^\top|Y]\cdot\E[\nabla\log p_X(X)|Y]\Vert^2] \\
&\leq \E[\Vert\E[\nabla\log p_X(X) \nabla\log p_X(X)^\top|Y]\Vert_{\text{op}}^2 \Vert\E[\nabla\log p_X(X)|Y]\Vert^2] \\
&\leq \E[(\E[\Vert\nabla\log p_X(X) \nabla\log p_X(X)^\top\Vert_{\text{op}}|Y])^2 \E[\Vert\nabla\log p_X(X)\Vert^2|Y]] \\
&\quad\text{(Jensen's inequality on the conditional expectation)} \\
&= \E[(\E[\Vert\nabla\log p_X(X)\Vert^2|Y])^2 \E[\Vert\nabla\log p_X(X)\Vert^2|Y]]\\
&= \E[(\E[\Vert \nabla\log p_X(X)\Vert^2|Y])^3] \\
&\leq \E[\Vert \nabla\log p_X(X)\Vert^6] \\
&\quad\text{(Jensen's inequality on the conditional expectation)} \\
&= C_1.
\end{align*}

Term (\ref{term:9}):
\begin{align*}  
\E[\Vert\Vert\E[\nabla\log p_X(X)|Y]\Vert^2&\E[\nabla \log p_X(X)|Y]\Vert^2]\\
&= \E[\Vert\E[\nabla \log p_X(X)|Y]\Vert^6] \\
&\leq \E[\Vert\nabla \log p_X(X)\Vert^6] \\
&\quad\text{(Jensen's inequality on the conditional expectation)} \\
&= C_1.
\end{align*}

To combine this bounds, we use Jensen's inequality on $x\mapsto x^2$, that gives for all $x_1,\dots,x_k \in \R$, $$\op\sum_{i=1}^k x_i\cp^2\leq k \sum_{i=1}^k x_i^2,$$ and we finally have,
$$
\E\ob\left\Vert \frac{d}{dt}\nabla \log p_t(x_t)\right\Vert^2\cb \leq \frac{9}{4}(4C_1 + (2d^2+5)C_1^{1/3}C_2^{2/3}+C_3) = C.
$$
\end{proof}

\begin{proof}(Proposition \ref{prop:W2half})
We have,
\begin{align*}
x_0 - \hat{x}_0 
&= \frac{1}{2}\int_0^t \nabla \log p_s(x_s) ds - \frac{1}{2} t \nabla \log p_t(x_t) \\
&= \frac{1}{2}\int_0^t \op\nabla \log p_s(x_s) - \nabla \log p_t(x_t)\cp ds \\
&= \frac{1}{2}\int_0^t \int_s^t \frac{d}{du}\op\nabla \log p_u(x_u) \cp du ds,
\end{align*}
hence, with Jensen's inequality,
$$
\E[\Vert x_0 - x_t\Vert^2] \leq \frac{t}{4} \int_0^t (t-s)\int_s^t \E\left[\left\Vert\frac{d}{du}\op\nabla \log p_u(x_u) \cp\right\Vert^2\right] du ds .
$$

Using Lemma \ref{lemma:derivative_bound}, it leads to $\E[\Vert x_0 - x_t\Vert^2] \leq \frac{C}{12}t^4$. In particular for $t= \sigma^2$, we get:
$$
W_2(\L(X),\L(\varphi_{1/2}(Y))) \leq (\E[\Vert x_0 - x_{\sigma^2}\Vert^2])^{1/2} \leq K \sigma^4,
$$
with $K = \sqrt{\frac{C}{12}} = \frac{\sqrt{3}}{4}\sqrt{4C_1 + (2d^2+5)C_1^{1/3}C_2^{2/3}+C_3}$.
\end{proof}

\subsection{Proof of Proposition \ref{prop:subspace}}
\begin{proof}
As the added noise $\ep$ is isotropic, it is invariant by rotation, thus, we can limit ourselves to the setting where $H = \R^m\times \{0\}^{d-m}$ for which the variable $X$ can be written $X = (X_1,0)$ with $X_1 \in \R^m$.

Write $\ep = (\ep_1,\ep_2)$, with  $\ep_1 \in \R^m \sim \N(0,\sigma^2I_{m}), \ep_2 \in \R^{(d-m)}\sim \N(0,\sigma^2I_{d-m})$ and $Y = (Y_1,Y_2) = (X_1 + \ep_1,\ep_2)$, with  $Y_1 \in \R^m, Y_2 \in \R^{(d-m)}$ and $Y_1\bot Y_2$. 
Hence, we have that 
$$\nabla \log p_Y(y) = (\nabla \log p_{Y_1}(y_1), \nabla \log p_{Y_2}(y_2)).$$
For the Gaussian $Y_2 = \ep_2 \sim \N(0,\sigma^2 I)$, $\nabla \log p_{Y_2}(y_2) = \frac{-y_2}{\sigma^2}$.

Therefore we have $\varphi_\alpha(Y) = (\varphi_\alpha(Y_1),(1-\alpha) \ep_2)$ and we can use the fact that for distributions $\mu_1,\mu_2,\nu_1,\nu_2$ $W_2^2(\mu_1\otimes\mu_2,\nu_1\otimes\nu_2) = W_2^2(\mu_1,\nu_1) + W_2^2(\mu_2,\nu_2)$, leading to
$$
W_2^2(\L(X),\L(\varphi_\alpha(Y))) = W_2^2(\L(X_1),\L(\varphi_\alpha(Y_1))) + W_2^2(0,(1-\alpha)\ep_2),
$$
with $W_2^2(\L(0),\L((1-\alpha)\ep_2)) = (d-m)(1-\alpha)^2\sigma^2$.

Note in particular that for $\alpha = 1$,
$$
\varphi_\alpha(Y) = (\varphi_\alpha(Y_1),(1-\alpha) \ep_2) = (\varphi_\alpha(Y_1),0) \in  H = \R^m\times \{0\}^{d-m}.
$$
Full-denoising hence ensures that the denoising variable $\varphi_\alpha(Y)$ belongs to the subspace $H$ as it remove all the orthogonal noise $\ep_2$.
\end{proof}

\subsection{Proof of Proposition \ref{prop:mixt}}
\begin{proof}
Let $\mu = \sum_{i=1}^N \pi_i \mu_i$ (with  $\sum_{i=1}^N \pi_i = 1,\pi_i \geq 0$ ) a mixture of distribution $\mu_i$ with compact support $S_i$ such that $D = \min_{i\neq j} d(S_i,S_j) > 0$ ($d(S_i,S_j) = \min_{x_i \in S_i, x_j \in S_j}\Vert x_i - x_j\Vert$).

We denote $X\sim \mu$, $Y = X + \ep$ with $X \bot \ep$ and $\ep \sim \N(0, \sigma^2 I)$, for $\sigma>0$. For $\alpha \in \R$, we denote $\varphi_\alpha(y) = y +\alpha \sigma^2 \nabla \log p_Y(y)$, $\nu$ the law of $Y$ and $\mu_\alpha$ the law of $\varphi_\alpha(Y)$. Similarly, for $X_i \sim \mu_i$, and  $Y_i = X_i + \ep$ with $X_i \bot \ep$ and $\ep \sim \N(0, \sigma^2 I)$,  we denote $\varphi_{i,\alpha}(y) = y +\alpha \sigma^2 \nabla \log p_{Y_i}(y)$, $\nu_i$ the law of $Y_i$ and $\mu_{i,\alpha}$ the law of $\varphi_{i,\alpha}(Y_i)$. We also denote  $\hat{\mu}_{i,\alpha}$ the law of $\varphi_\alpha(Y_i)$.

We denote $R = \sup_{i,x\in S_i} \Vert x\Vert < \infty$. We have
$$\varphi_\alpha(y) = y + \alpha\sigma^2\nabla\log p_Y(y) = (1-\alpha)y + \alpha\E[X|Y=y],$$
and, 
$$\varphi_{i,\alpha}(y) = y + \alpha\sigma^2\nabla\log p_{Y_i}(y) = (1-\alpha)y + \alpha\E[X_i|Y_i=y].$$
Note that for all $y$, $\Vert\E[X|Y=y]\Vert \leq R$ and $\Vert\E[X_i|Y_i=y]\Vert \leq R$.

By limiting ourselves to the couplings $\Gamma^0 (\mu,\mu_\alpha) = \{ \gamma: (X,\hat{X})\sim \gamma \text{ with } X = \sum_i 1_{Z=i} X_i, \hat{X} = \sum 1_{Z=i} \hat{X}_i, Z \text{ such  that } P(Z=i) = \pi_i, \text{ and } (X_i,\hat{X}_i) \sim \gamma_i \in \Gamma (\mu_i, \hat{\mu}_{i,\alpha})\},$ we have:
\begin{align*}
W_2^2(\mu,\mu_\alpha) &= \inf_{\gamma \in \Gamma(\mu,\mu_\alpha)} \int \Vert x - \hat{x}\Vert^2 d \gamma(x,\hat{x}) \\
&\leq \inf_{\gamma \in \Gamma^0(\mu,\mu_\alpha)} \int \Vert x - \hat{x}\Vert^2 d \gamma(x,\hat{x}) \\
&= \sum_i \pi_i \inf_{\gamma \in \Gamma (\mu_i, \hat{\mu}_{i,\alpha})} \int \Vert x - \hat{x}\Vert^2 d \gamma(x,\hat{x}) \\
&= \sum_i \pi_i W_2^2(\mu_i,\hat{\mu}_{i,\alpha}).
\end{align*}

Then for $i \in \{1,\dots,N\}$ we have:
\begin{align*}
W_2^2(\mu_i,\hat{\mu}_{i,\alpha}) &= \inf_{\gamma \in \Gamma(\mu_i,\hat{\mu}_{i,\alpha})} \int \Vert x - \hat{x}\Vert^2 d \gamma(x,\hat{x}) \\
&= \inf_{\gamma \in \Gamma(\mu_i,\nu_i)} \int \Vert x - \varphi_\alpha(y)\Vert^2 d \gamma(x,y) \\
&\leq 2\inf_{\gamma \in \Gamma(\mu_i,\nu_i)} \int \Vert x - \varphi_{i,\alpha}(y)\Vert^2 d \gamma(x,y) + \int \Vert \varphi_{i,\alpha}(y) - \varphi_\alpha(y)\Vert^2 d \gamma(x,y) \\
&= 2 W_2^2(\mu_i,\mu_{i,\alpha}) + \int \Vert \varphi_{i,\alpha}(y) - \varphi_\alpha(y)\Vert^2 d \nu_i(y).
\end{align*}

To conclude, it is sufficient to prove that $\int \Vert \varphi_{i,\alpha}(y) - \varphi_\alpha(y)\Vert^2 d \nu_i(y) = O\op\frac{1}{\sigma^{d-2}}\exp\op-\frac{K}{\sigma^2}\cp\cp$ with  $K = K(D)$ a constant. We fix $\delta_1 > 0$ such that $D-\delta_1 > 0$. We have,
\begin{align*}
\int \Vert \varphi_{i,\alpha}(y) - \varphi_\alpha(y)\Vert^2 d \nu_i(y) &= \int_{y \in S_i + B(0,\delta_1)}\Vert \varphi_{i,\alpha}(y) - \varphi_\alpha(y)\Vert^2 d \nu_i(y) \\
&+ \int_{y \notin S_i + B(0,\delta_1)} \Vert \varphi_{i,\alpha}(y) - \varphi_\alpha(y)\Vert^2 d \nu_i(y).
\end{align*}

We start by bounding the term $\int_{y \notin S_i + B(0,\delta_1)} \Vert \varphi_{i,\alpha}(y) - \varphi_\alpha(y)\Vert^2 d \nu_i(y).$
\begin{align*}
\int_{y \notin S_i + B(0,\delta_1)} \Vert \varphi_{i,\alpha}(y) - \varphi_\alpha(y)\Vert^2 d \nu_i(y) &= \E [\Vert \varphi_{i,\alpha}(Y_i) - \varphi_\alpha(Y_i)\Vert^2 1_{Y_i \notin S_i + B(0,\delta_1)}] \\
&= \alpha^2 \E [\Vert \E[X_i|Y_i] - \E[X|Y=Y_i] \Vert^2 1_{Y_i \notin S_i + B(0,\delta_1)} ] \\
&\leq 4R^2\alpha^2 P(Y_i \notin S_i + B(0,\delta_1)).
\end{align*}

We have $Y_i = X_i +\ep$ with $\ep \sim \N(0,\sigma^2)$ and $\frac{\Vert \ep\Vert^2}{\sigma^2} \sim \chi^2$ (chi-squared distribution). More over, as $\text{supp}(X_i) \subset S_i$,  $\{Y_i \notin S_i + B(0,\delta_1)\} \subset \{ \ep \notin B(0,\delta_1)\} = \left\{ \frac{\Vert \ep\Vert^2}{\sigma^2} > \frac{\delta_1^2}{\sigma^2}\right\}$. Therefore:

\begin{align*}
P(Y_i \notin S_i + B(0,\delta_1))  &\leq P\op\frac{\Vert \ep\Vert^2}{\sigma^2} > \frac{\delta_1^2}{\sigma^2}\cp \\
&= \frac{(1/2)^{d/2}}{\Gamma(d/2)} \int_{\frac{\delta_1^2}{\sigma^2}}^\infty t^{d/2 - 1}e^{-t/2}dt\\
&\leq \frac{4 (1/2)^{d/2}\delta_1^{d-2}}{\Gamma(d/2)} \frac{1}{\sigma^{d-2}} \exp\op- \frac{\delta_1^2}{2\sigma^2}\cp\\
&= O\op \frac{1}{\sigma^{d-2}}\exp\op- \frac{\delta_1^2}{2\sigma^2}\cp\cp \text{ when } \sigma \rightarrow 0,
\end{align*}
where $\Gamma$ is the Gamma function defined by $\Gamma(z) = \int_0^\infty t^{z-1} e^{-t}\,dt$, for $\Re(z) >0$. This leads to
\begin{align*}
\int_{y \notin S_i + B(0,\delta_1)} \Vert \varphi_{i,\alpha}(y) - \varphi_\alpha(y)\Vert^2 d \nu_i(y)
&\leq \frac{16R^2\alpha^2(1/2)^{d/2}\delta_1^{d-2}}{\Gamma(d/2)} \frac{1}{\sigma^{d-2}} \exp\op- \frac{\delta_1^2}{2\sigma^2}\cp\\
&= O\op \frac{1}{\sigma^{d-2}}\exp\op- \frac{\delta_1^2}{2\sigma^2}\cp\cp \text{ when } \sigma \rightarrow 0.
\end{align*}

We now turn to the term $\int_{y \in S_i + B(0,\delta_1)}\Vert \varphi_{i,\alpha}(y) - \varphi_\alpha(y)\Vert^2 d \nu_i(y) .$
We fix $\delta_2 >0$ such that $(D-\delta_1)^2 -\delta_2 > 0$ and we denote $A = \{p_{Y_i}(y) \leq \frac{1}{\pi_i \sigma^{d-2}}\exp\op-\frac{(D-\delta_1)^2 - \delta_2}{2\sigma^2}\cp\}$ and  $B = \{p_{Y_i}(y) \geq \frac{1}{\pi_i\sigma^{d-2}}\exp\op-\frac{(D-\delta_1)^2 - \delta_2}{2\sigma^2}\cp\}$.
On $A$, we have
\begin{align*}
 \int_{y \in (S_i + B(0,\delta_1))\cap A}&\Vert \varphi_{i,\alpha}(y) - \varphi_\alpha(y)\Vert^2 d \nu_i(y) \\
 &= \int_{y \in (S_i + B(0,\delta_1))\cap A}\Vert \varphi_{i,\alpha}(y) - \varphi_\alpha(y)\Vert^2 p_{Y_i}(y) dy \\
 &=\alpha^2 \int_{y \in (S_i + B(0,\delta_1))\cap A}\Vert \E[X_i|Y_i=y] + \E[X|Y=y] \Vert^2  d \nu_i(y) \\
 &\leq \alpha^2 \int_{y \in (S_i + B(0,\delta_1))\cap A} \frac{4R^2}{\pi_i\sigma^{d-2}} \exp\op\frac{(D-\delta_1)^2 - \delta_2}{2\sigma^2}\cp dy \\
 &\leq \frac{4\alpha^2R^2}{\pi_i\sigma^{d-2}} \exp\op\frac{(D-\delta_1)^2 - \delta_2}{2\sigma^2}\cp \mathit{Vol}(S_i + B(0,\delta_1))\\ 
 &\leq \frac{4 \alpha^2R^2(R+\delta_1)^{d} }{\pi_i\sigma^{d-2}} \exp\op\frac{(D-\delta_1)^2 - \delta_2}{2\sigma^2}\cp.
\end{align*}

To bound the term on B, we first write $p_Y(y) = \sum_j \pi_j p_{Y_j}(y) =  \pi_i p_{Y_i}(y) + f_i(y)$ with $f_i(y) = \sum_{j\neq i} \pi_j p_{Y_j}(y)$, and we notice that for $y \in S_i + B(0,\delta_1)$, we have: 
\begin{align*}
f_i(y) &= \sum_{j\neq i} \pi_j \int \frac{1}{C\sigma^d} \exp\op -\frac{\Vert y-x\Vert^2}{2\sigma^2}\cp d \mu_j (x)\\
&\leq \frac{1}{C\sigma^d}\exp\op\frac{-(D-\delta_1)^2}{2\sigma^2}\cp,
\end{align*}
with $C = (2\pi)^{d/2}$, as well as,
\begin{align*}
\Vert \nabla f_i(y)\Vert &= \left\Vert \sum_{j\neq i} \pi_j \int \frac{y-x}{C\sigma^2} \exp\op -\frac{\Vert y-x\Vert^2}{2\sigma^{d+2}}\cp d \mu_j (x) \right\Vert\\
&\leq \frac{2(R+\delta_1)}{C\sigma^{d+2}}\exp\op\frac{-(D-\delta_1)^2}{2\sigma^2}\cp,
\end{align*}
and,
$$
\Vert \nabla \log p_{Y_i}(y)\Vert = \left\Vert \frac{1}{\sigma^2} (\E[X_i|Y_i=y]-y)\right\Vert\leq \frac{2(R+\delta_1)}{\sigma^2}.
$$

Therefore, for $y \in (S_i + B(0,\delta_1))\cap B$:
\begin{align*}
\Vert \nabla \log p_Y(y) - \nabla \log p_{Y_i}(y) \Vert &= \left\Vert \frac{\pi_i \nabla p_{Y_i}(y) + \nabla f_i(y)}{\pi_i p_{Y_i}(y) + f_i(y)} - \nabla \log p_{Y_i}(y) \right\Vert\\
&= \left\Vert \nabla \log p_{Y_i}(y)\op\frac{1}{1 + \frac{f_i(y)}{\pi_i p_{Y_i}(y)}} -1\cp + \frac{\nabla f_i(y)}{\pi_i p_{Y_i}(y) + f_i(y)}  \right\Vert\\
&\leq \frac{f_i(y)}{\pi_i p_{Y_i}(y)}\Vert \nabla \log p_{Y_i}(y) \Vert + \frac{\Vert\nabla f_i(y)\Vert}{\pi_i p_{Y_i}(y)} \\
& \text{ (for }t>0, \left|\frac{1}{1+t} -1\right| =  t\left|\frac{1}{1+t}\right|\leq t) \\
&\leq \frac{1}{C\sigma^{2}}\exp\op\frac{-(D-\delta_1)^2}{2\sigma^2} + \frac{(D-\delta_1)^2 -\delta_2}{2\sigma^2}\cp \frac{2(R+\delta_1)}{\sigma^2} \\
&\quad + \frac{2(R+\delta_1)}{C\sigma^{4}}\exp\op\frac{-(D-\delta_1)^2}{2\sigma^2} + \frac{(D-\delta_1)^2 -\delta_2}{2\sigma^2}\cp \\
&\leq \frac{4(R+\delta_1)}{C\sigma^{4}}\exp\op-\frac{\delta_2}{2\sigma^2}\cp.
\end{align*}

This leads to:
\begin{align*}
\int_{y \in (S_i + B(0,\delta_1))\cap B}&\Vert \varphi_{i,\alpha}(y) - \varphi_\alpha(y)\Vert^2 d \nu_i(y) \\
&= \alpha^2 \sigma^4 \int_{y \in (S_i + B(0,\delta_1))\cap B}\Vert \nabla \log p_Y(y) - \nabla \log p_{Y_i}(y) \Vert^2  d \nu_i(y)\\
&\leq \alpha^2 \sigma^4\int_{y \in (S_i + B(0,\delta_1))\cap B} \frac{16(R+\delta_1)^2}{C^2\sigma^{8}}\exp\op-\frac{\delta_2}{\sigma^2}\cp  d \nu_i(y)\\
&\leq \frac{16\alpha^2(R+\delta_1)^2}{C^2\sigma^{8}} \exp\op-\frac{\delta_2}{\sigma^2}\cp.
\end{align*}

Finally, we have:
\begin{align*}
\int \Vert \varphi_{i,\alpha}(y) &- \varphi_\alpha(y)\Vert^2 d \nu_i(y) \\
&\leq \frac{16R^2\alpha^2(1/2)^{d/2}\delta_1^{d-2}}{\Gamma(d/2)} \frac{1}{\sigma^{d-2}} \exp\op- \frac{\delta_1^2}{2\sigma^2}\cp \\
&\quad + \frac{4 \alpha^2R^2(R+\delta_1)^{d} }{\pi_i\sigma^{d-2}} \exp\op\frac{(D-\delta_1)^2 - \delta_2}{2\sigma^2}\cp \\
&\quad + \frac{16\alpha^2(R+\delta_1)^2}{C^2\sigma^{8}} \exp\op-\frac{\delta_2}{\sigma^2}\cp\\
&= O\op \frac{1}{\sigma^{d-2}} \exp\op- \frac{\delta_1^2}{2\sigma^2}\cp + \frac{1}{\sigma^{d-2}}\exp\op-\frac{(D-\delta_1)^2-\delta_2}{2\sigma^2}\cp + \frac{1}{\sigma^{8}}\exp\op-\frac{\delta_2}{\sigma^2}\cp\cp \\
&= O\op\frac{1}{\sigma^{\max(d-2,8)}}\exp\op-\frac{K}{\sigma^{2}}\cp\cp,
\end{align*}
with $K = \min\op\frac{\delta_1^2}{2},\frac{(D-\delta_1)^2 -\delta_2}{2},\delta_2\cp > 0$. (For example take $\delta_1 = \frac{D}{2}$ and $\delta_2 = \frac{D^2}{8}$ to get $K = \frac{D^2}{16}$.)
\end{proof}

\textbf{Remark:} we also have,
\begin{align*}
W_2^2(\mu,\mu_\alpha) - 2 \sum_i \pi_i W_2^2(\mu_i,\mu_{i,\alpha})
&\leq \frac{32R^2\alpha^2(1/2)^{d/2}\delta_1^{d-2}}{\Gamma(d/2)} \frac{1}{\sigma^{d-2}} \exp\op- \frac{\delta_1^2}{2\sigma^2}\cp \\
&\quad + \frac{8N\alpha^2R^2(R+\delta_1)^{d}}{\sigma^{d-2}} \exp\op\frac{(D-\delta_1)^2 - \delta_2}{2\sigma^2}\cp \\
&\quad + \frac{32\alpha^2(R+\delta_1)^2}{C^2\sigma^8} \exp\op-\frac{\delta_2}{\sigma^2}\cp\\
&= O\op\frac{1}{\sigma^{\max(d-2,8)}}\exp\op-\frac{K}{\sigma^{2}}\cp\cp.
\end{align*}

\section{Usual distributions Verify the Hypothesis of Lemma \ref{lemma:derivative_bound}}
\label{sct:usual_distrib}
\begin{proposition}
Assume that  $X = Z + \ep_0$ , with  $\E[\Vert Z\Vert^6] <\infty$, $Z \bot \ep_0$ and $\ep_0 \sim \N(0,\tau^2)$, then $X$ verifies the assumptions of Lemma \ref{lemma:derivative_bound} with the constants:
$$
\begin{array}{lll}
C_1 &= \E[\Vert \nabla \log p_X(X)\Vert^6] &\leq \frac{243}{\tau^{12}} (2\E[\Vert Z\Vert^6] + 15d\tau^6),\\
C_2 &= \E[\Vert \nabla^2\log p_X(X)\Vert_{\textnormal{op}}^3] &\leq 9 \op \frac{1}{\tau^6} + \frac{2\E[\Vert Z\Vert^6]}{\tau^{12}}\cp,\\
C_3 &= \E[\Vert \nabla \Delta \log p_X(X)\Vert^2]&\leq \frac{40\E[\Vert Z\Vert^6]}{\tau^{12}}.
\end{array}
$$
\end{proposition}
\begin{proof}
First note that $p_X: x \mapsto \frac{1}{(2\pi \tau^2)^{d/2}}\int e^{\frac{\Vert x-z\Vert^2}{2\tau^2}}d\mu(z) \in \C^\infty(\R^d)$. Then (B.1) and (B.3) of \citet{saremiChainLogConcaveMarkov2023} give
$$
\nabla\log p_X (x) = \frac{1}{\tau^2}(\E[Z|X = x] - x),
$$
and 
\begin{align*}
\nabla^2\log p_X(x) &= - \frac{1}{\tau^2}I + \frac{1}{\tau^4}\op\E[ZZ^\top|X=x] -\E[Z|X=x]\E[Z|X=x]^\top\cp \\
&=- \frac{1}{\tau^2}I + \frac{1}{\tau^4}\text{cov}(Z|X=x).
\end{align*}
Similarly, we can compute,
\begin{align*}
\nabla\Delta  \log p_X(x) &= \frac{-1}{\tau^6} \E[\Vert Z\Vert^2Z|X=x]\\
&\quad+ \frac{1}{\tau^6}  \E[\Vert Z\Vert^2|X=x]\E[Z|X=x]\\
&\quad- \frac{2}{\tau^6} \E[ZZ^\top|X=x] \cdot\E[Z|X = x]\\
&\quad+ \frac{2}{\tau^6} \Vert\E[Z|X = x]\Vert^2 \E[Z|X = x].
\end{align*}

We can now bound the constants $C_1$, $C_2$ and $C_3$. Using Jensen's inequality on $x\mapsto x^p$, that gives for all $x_1,\dots,x_k \in \R$, $$\op\sum_{i=1}^k x_i\cp^p \leq k^{p-1} \sum_{i=1}^k x_i^p,$$ and we have,
\begin{align*}
C_1 = \E[\Vert \nabla \log p_X(X)\Vert^6] &= \E\ob\left\Vert\frac{1}{\tau^2}(\E [ Z| X] - X )\right\Vert^6\cb\\
&= \E\ob\left\Vert\frac{1}{\tau^2}(\E [ Z| X] - Z - \ep_0)\right\Vert^6\cb\\
&\leq \frac{243}{\tau^{12}} (\E[\Vert\E [ Z| X]\Vert^6] + \E[\Vert Z\Vert^6] + \E[\Vert \ep_0\Vert^6]) \\
&\leq \frac{243}{\tau^{12}} (\E[\E[\Vert Z\Vert^6| X]] + \E[\Vert Z\Vert^6] + 15d\tau^6) \\
&\quad \text{(Jensen's inequality on conditional expectation)} \\
&= \frac{243}{\tau^{12}} (2\E[\Vert Z\Vert^6] + 15d\tau^6),
\end{align*}
and,
\begin{align*}
C_2 &= \E[\Vert \nabla \log p_X(X)\Vert_{\text{op}}^3]\\
&= \E\ob \left\Vert- \frac{1}{\tau^2}I + \frac{1}{\tau^4}\op\E[ZZ^\top|X] -\E[Z|X]\E[Z|X]^\top\cp \right\Vert_{\text{op}}^3\cb \\
&\leq 9 \op \frac{\Vert I \Vert^3_{\text{op}}}{\tau^6} +\frac{1}{\tau^{12}}\E[\Vert \E[ZZ^\top|X]\Vert^3_{\text{op}}]  + \frac{1}{\tau^{12}}\E[\Vert\E[Z|X]\E[Z|X]^\top\Vert^3_{\text{op}}]\cp \\
&= 9 \op \frac{1}{\tau^6} +\frac{1}{\tau^{12}}\E[\Vert \E[ZZ^\top|X]\Vert^3_{\text{op}}]  + \frac{1}{\tau^{12}}\E[\Vert\E[Z|X]\Vert^6]\cp\\
&\leq 9 \op \frac{1}{\tau^6} +\frac{1}{\tau^{12}}\E[\Vert ZZ^\top\Vert^3_{\text{op}}]  + \frac{1}{\tau^{12}}\E[\Vert Z\Vert^6]\cp\\
&\quad \text{(Jensen's inequality on conditional expectation)} \\
&= \leq 9 \op \frac{1}{\tau^6} + \frac{2\E[\Vert Z\Vert^6]}{\tau^{12}}\cp,
\end{align*}
and finally we control the 4 terms in $\nabla\Delta  \log p_X$,
\begin{align*}
\E[\Vert\E[\Vert Z\Vert^2Z|X]\Vert^2] &\leq \E[\Vert\Vert Z\Vert^2Z\Vert^2]\\ &\quad\text{(Jensen's inequality on conditional expectation)}\\
&= \E[\Vert Z\Vert^6]\\
\E[\Vert\E[\Vert Z\Vert^2|X]\E[Z|X]\Vert^2] &= \E[(E[\Vert Z\Vert^2|X])^2\Vert\E[Z|X]\Vert^2] \\
&\leq  \E[(E[\Vert Z\Vert^2|X])^2\E[\Vert Z\Vert^2|X]] =  \E[(E[\Vert Z\Vert^2|X])^3]\\
&\quad\text{(Jensen's inequality on conditional expectation)}\\
&\leq \E[\Vert Z\Vert^6]\\
&\quad\text{(Jensen's inequality on conditional expectation)}\\
\E[\Vert \E[ZZ^\top|X] \cdot\E[Z|X]\Vert^2] &\leq \E[\Vert \E[ZZ^\top|X]\Vert^2_\text{op}\Vert\E[Z|X]\Vert^2]\\
&\leq \E[(\E[\Vert ZZ^\top\Vert_\text{op}|X])^2\E[\Vert Z\Vert^2|X]]\\
&=  \E[(E[\Vert Z\Vert^2|X])^3]\\
&\leq \E[\Vert Z\Vert^6]\\
&\quad\text{(Jensen's inequality on conditional expectation)}\\
\E[\Vert\Vert\E[Z|X]\Vert^2 \E[Z|X]\Vert^2] &= \E[\Vert\E[Z|X]\Vert^6] \\
&\leq \E[\Vert Z\Vert^6]\\
&\quad\text{(Jensen's inequality on conditional expectation)},
\end{align*}

to get,
$$
C_3 \leq \frac{40\E[\Vert Z\Vert^6]}{\tau^{12}}.
$$
\end{proof}

\section{Extension of Results from Section \ref{sct:wasserstein} to any \texorpdfstring{$p \geq 1$}{p greater than 1}}
\label{sct:extended_wasserstein}
We extend all results from section \ref{sct:wasserstein} to Wasserstein-$p$ distances for any $p \geq 1$.

\begin{proposition}
Assume that $\E[\Vert \nabla \log p_X(X)\Vert^p]\leq C$. Then
$$
W_p(\L(X),\L(\varphi_\alpha(Y))) \leq \op\frac{(1+2^p\alpha^p)C}{2}\cp^{1/p} \sigma^2.
$$
\end{proposition}

\begin{proof}
We have 
$$x_0 =  x_t + \frac{1}{2}\int_0^t \nabla \log p_s(x_s) ds,$$
and 
$$\hat{x}_0 = x_t + \alpha t \nabla \log p_t(x_t).$$
Therefore,
$$
x_0 - \hat{x}_0 = \frac{1}{2}\int_0^t \nabla \log p_s(x_s) ds - \alpha t \nabla \log p_t(x_t).
$$

Using Jensen's inequality two times,
\begin{align*}
\Vert x_0 - \hat{x}_0\Vert^p
&\leq 2^{p-1}\op\left\Vert\frac{1}{2}\int_0^t \nabla \log p_s(x_s) ds\right\Vert^p  + \Vert\alpha t \nabla \log p_t(x_t)\Vert^p\cp\\
&\leq 2^{p-1}\op\frac{t^{p-1}}{2^p}\int_0^t \Vert\nabla \log p_s(x_s)\Vert^pds + \alpha^p t^p \Vert\nabla \log p_t(x_t)\Vert^p\cp,  
\end{align*}
and it follows that,
\begin{align*}
W_p^p(\L(X),\L(\varphi_\alpha(X_t))) &\leq  \E[\Vert x_0 - x_t\Vert^p] \\
&\leq 2^{p-1}\op\frac{t^{p-1}}{2^p}\int_0^t \E[\Vert\nabla \log p_s(x_s)\Vert^p]ds + \alpha^p t^p \E[\Vert\nabla \log p_t(x_t)\Vert^p]\cp.
\end{align*}

As stated in Lemma \ref{lemma:expectation_gradient_noise}, assuming $\E[\Vert\nabla \log p_X(X)\Vert^p] \leq C$ leads to $\E[\Vert\nabla \log p_s(x_s)\Vert^p] \leq C$ for all $s\geq 0$, so finally:
$$
W_p^p(\L(X),\L(\varphi_\alpha(X_t))) \leq \frac{(1+2^p\alpha^p)C}{2}t^p,
$$
which gives for $t = \sigma^2$:
$$
W_p (\L(X),\L(\varphi_\alpha(Y))) \leq \op\frac{(1+2^p\alpha^p)C}{2}\cp^{1/p} \sigma^2.
$$
\end{proof}

\begin{lemma}
\label{lemma:derivative_bound_p}
Assume that the variable $X \sim \mu$ of density $p_X$ verifies that:
\begin{itemize}
    \item $\log p_X \in \mathcal{C}^3(\R^d)$. 
    \item $C_1 = \E[\Vert \nabla \log p_X(X)\Vert^{3p}] < \infty$.
    \item $C_2 = \E[\Vert \nabla^2\log p_X(X)\Vert_{\textnormal{op}}^{3p/2}] < \infty$, where $\Vert A\Vert_\textnormal{op}$ is defined for any matrix $A$ as $\Vert A \Vert_\textnormal{op} = \sup_{x\neq 0}\frac{\Vert Ax\Vert}{\Vert x \Vert}$.
    \item $C_3 = \E[\Vert \nabla \Delta \log p_X(X)\Vert^{p}] < \infty$.
\end{itemize}
Then,
$$
\E\ob\left\Vert \frac{d}{dt}\nabla \log p_t(x_t)\right\Vert^p\cb \leq C,
$$
with $C = \frac{9^{p-1}}{2^p}(3C_1 + 2(d^2+1)C_1^{1/3}C_2^{2/3}+C_3)$.
\end{lemma}

\begin{proof}
Once again, we want to control all the terms (\ref{term:1}-\ref{term:9}), but this time in $L_p$ norm. Recall that for all $x \in \R^d$, $|\Delta  \log p_X(x)| \leq d\Vert \nabla^2 \log p_X(x)\Vert_\text{op}$, hence, $\E[|\Delta  \log p_X(X)|^{3p/2}]\leq d^{3p/2}C_2$.

Term (\ref{term:1}):
\begin{align*} 
\E[\Vert\E[\nabla\Delta  \log p_X(X)|Y]\Vert^p] 
&\leq \E[\E[\Vert\nabla\Delta  \log p_X(X)\Vert^p|Y]] = \E[\Vert\nabla\Delta  \log p_X(X)\Vert^p]\\
&\quad\text{(Jensen's inequality on the conditional expectation)} \\
&= C_3.
\end{align*}

Term (\ref{term:2}):
\begin{align*} 
\E[\Vert\E[\Delta  \log p_X(X) \nabla\log p_X(X) |Y]\Vert^p] 
&\leq \E[|\Delta  \log p_X(X)|^p\Vert\nabla\log p_X(X)\Vert^p] \\
&\quad\text{(Jensen's inequality on the conditional expectation)} \\
&\leq \E[|\Delta  \log p_X(X)|^{3p/2}]^{2/3} \E[\Vert\nabla\log p_X(X)\Vert^{3p}]^{1/3}\\
&\quad\text{(Hölder's inequality)} \\
&\leq d^p C_2^{2/3}C_1^{1/3}.
\end{align*}

Term (\ref{term:3}):
\begin{align*} 
\E[\Vert\E[\Delta  \log p_X(X)|Y]&\E[\nabla\log p_X(X)|Y]\Vert^p] \\
&= \E[|\E[\Delta  \log p_X(X)|Y]|^p \Vert\E[\nabla\log p_X(X)|Y]\Vert^p] \\
&\leq \E[|\E[\Delta  \log p_X(X)|Y]|^{3p/2}]^{2/3} \E[\Vert\E[\nabla\log p_X(X)|Y]\Vert^{3p}]^{1/3} \\
&\quad\text{(Hölder's inequality)} \\
&\leq \E[|\Delta  \log p_X(X)|^{3p/2}]^{2/3} \E[\Vert\nabla\log p_X(X)\Vert^{3p}]^{1/3} \\
&\quad\text{(Jensen's inequality on the conditional expectation)} \\
&\leq d^p C_2^{2/3}C_1^{1/3}.
\end{align*}

Term (\ref{term:4}):
 \begin{align*} 
 \E[\Vert\E[\nabla^2\log p_X(X)&\cdot\nabla\log p_X(X)|Y]\Vert^p] \\
&\leq \E[\Vert\nabla^2\log p_X(X)\cdot\nabla\log p_X(X)\Vert^p]\\
&\quad\text{(Jensen's inequality on the conditional expectation)} \\
&\leq \E[\Vert\nabla^2\log p_X(X)\Vert_\text{op}^p \Vert\nabla\log p_X(X)\Vert^p]\\
&\leq \E[\Vert\nabla^2\log p_X(X)\Vert_\text{op}^{3p/2}]^{2/3} \E[\Vert\nabla\log p_X(X)\Vert^{3p}]^{1/3}\\
&\quad\text{(Hölder's inequality)} \\
&= C_2^{2/3}C_1^{1/3}.
 \end{align*}
 
Term (\ref{term:5}):
\begin{align*} 
\E[\Vert\E[\Vert \nabla\log p_X(X)\Vert^2 &\nabla\log p_X(X)|Y]\Vert^p]\\
&\leq \E[\Vert \nabla\log p_X(X)\Vert^{3p}] \\
&\quad\text{(Jensen's inequality on the conditional expectation)} \\
&= C_1.
\end{align*}

Term (\ref{term:6}):
\begin{align*} 
\E[\Vert\E[\Vert \nabla\log p_X(X)\Vert^2|Y]&\E[\nabla\log p_X(X)|Y]\Vert^p]\\
&=\E[(\E[\Vert \nabla\log p_X(X)\Vert^2|Y])^p \Vert\E[\nabla\log p_X(X)|Y]\Vert^p]\\
&\leq \E[(\E[\Vert \nabla\log p_X(X)\Vert^2|Y])^p (\E[\Vert\nabla\log p_X(X)\Vert^2|Y])^{p/2}] \\
&\quad\text{(Jensen's inequality on the conditional expectation)} \\
&= \E[(\E[\Vert \nabla\log p_X(X)\Vert^2|Y])^{3p/2}] \\
&\leq \E[\Vert \nabla\log p_X(X)\Vert^{3p}] \\
&\quad\text{(Jensen's inequality on the conditional expectation)} \\
&= C_1.
\end{align*}

Term (\ref{term:7}):
\begin{align*} 
\E[\Vert\E[\nabla^2\log p_X(X)|Y]&\cdot\E[\nabla\log p_X(X)|Y]\Vert^p]\\
&\leq \E[\Vert\E[\nabla^2\log p_X(X)|Y]\Vert^p_\text{op}\Vert\E[\nabla\log p_X(X)|Y]\Vert^p] \\
&\leq \E[\Vert\E[\nabla^2\log p_X(X)|Y]\Vert^{3p/2}_\text{op}]^{2/3} \E\Vert\E[\nabla\log p_X(X)|Y]\Vert^{3p}]^{1/3} \\
&\quad\text{(Hölder's inequality)} \\
&\leq \E[\Vert\nabla^2\log p_X(X)\Vert^{3p/2}_\text{op}]^{2/3} \E\Vert\nabla\log p_X(X)\Vert^{3p}]^{1/3} \\
&\quad\text{(Jensen's inequality on the conditional expectation)} \\
&= C_2^{2/3}C_1^{1/3}.
\end{align*}

Term (\ref{term:8}):
\begin{align*} 
\E[\Vert\E[\nabla\log p_X(X) &\nabla\log p_X(X)^\top|Y]\cdot\E[\nabla\log p_X(X)|Y]\Vert^p] \\
&\leq \E[\Vert\E[\nabla\log p_X(X) \nabla\log p_X(X)^\top|Y]\Vert_{\text{op}}^p \Vert\E[\nabla\log p_X(X)|Y]\Vert^p] \\
&\leq \E[(\E[\Vert\nabla\log p_X(X) \nabla\log p_X(X)^\top\Vert_{\text{op}}|Y])^p (\E[\Vert\nabla\log p_X(X)\Vert^2|Y])^{p/2}] \\
&\quad\text{(Jensen's inequality on the conditional expectation)} \\
&= \E[(\E[\Vert\nabla\log p_X(X)\Vert^2|Y])^p \Vert\E[\nabla\log p_X(X)|Y]\Vert^p]\\
&= \E[(\E[\Vert \nabla\log p_X(X)\Vert^2|Y])^{3p/2}] \\
&\leq \E[\Vert \nabla\log p_X(X)\Vert^{3p}] \\
&\quad\text{(Jensen's inequality on the conditional expectation)} \\
&= C_1.
\end{align*}

Term (\ref{term:9}):
\begin{align*}  
\E[\Vert\Vert\E[\nabla\log p_X(X)|Y]\Vert^2&\E[\nabla \log p_X(X)|Y]\Vert^p] \\
&= \E[\Vert\E[\nabla \log p_X(X)|Y]\Vert^{3p}] \\
&\leq \E[\Vert\nabla \log p_X(X)\Vert^{3p}] \\
&\quad\text{(Jensen's inequality on the conditional expectation)} \\
&= C_1.
\end{align*}

To combine this bounds, we use Jensen's inequality on $x\mapsto x^p$, that gives for all $x_1,\dots,x_k \in \R$, $$\op\sum_{i=1}^k x_i\cp^p \leq k^{p-1} \sum_{i=1}^k x_i^p,$$ and we finally have,
$$
\E\ob\left\Vert \frac{d}{dt}\nabla \log p_t(x_t)\right\Vert^2\cb \leq \frac{9^{p-1}}{2^p}(4C_1 + (2d^2+5)C_1^{1/3}C_2^{2/3}+C_3) = C.
$$
\end{proof}

\begin{proposition}
\label{prop:W2half_p}
Under the same assumptions of Lemma \ref{lemma:derivative_bound}, we have 
$$
W_2(\L(X),\L(\varphi_{1/2}(Y))) \leq K \sigma^4,
$$
with $K = \frac{9^{(p-1)/p}}{4(p+1)^{1/p}}(4C_1 + (2d^2+5)C_1^{1/3}C_2^{2/3}+C_3)^{1/p}$.
\end{proposition}

\begin{proof}
We have,
\begin{align*}
x_0 - \hat{x}_0 
&= \frac{1}{2}\int_0^t \nabla \log p_s(x_s) ds - \frac{1}{2} t \nabla \log p_t(x_t) \\
&= \frac{1}{2}\int_0^t \op\nabla \log p_s(x_s) - \nabla \log p_t(x_t)\cp ds \\
&= \frac{1}{2}\int_0^t \int_s^t \frac{d}{du}\op\nabla \log p_u(x_u) \cp du ds.
\end{align*}
hence, with Jensen's inequality,
$$
\E[\Vert x_0 - x_t\Vert^p] \leq \frac{t^{p-1}}{2^p} \int_0^t (t-s)^{p-1}\int_s^t \E\left[\left\Vert\frac{d}{du}\op\nabla \log p_u(x_u) \cp\right\Vert^p\right] du ds ,
$$

Using Lemma \ref{lemma:derivative_bound}, it leads to $\E[\Vert x_0 - x_t\Vert^2] \leq \frac{C}{2^p(p+1)}t^{2p}$. In particular for $t= \sigma^2$, we get:
$$
W_p(\L(X),\L(\varphi_{1/2}(Y))) \leq \op\E[\Vert x_0 - x_{\sigma^2}\Vert^p]\cp^{1/p} \leq K \sigma^{4},
$$
with $K = \op\frac{C}{2^p(p+1)}\cp^{1/p} = \frac{9^{(p-1)/p}}{4(p+1)^{1/p}}(4C_1 + (2d^2+5)C_1^{1/3}C_2^{2/3}+C_3)^{1/p}$.
\end{proof}

Remark: Proposition \ref{prop:mixt} from section \ref{sct:mixt} can also be extended for any $p\geq1$ (tough it will make the proof harder to read). However Proposition \ref{prop:subspace} from section \ref{sct:subspace} relies on the fact that $\Vert (x_1,x_2)\Vert^2 = \Vert x_1\Vert^2 + \Vert x_2\Vert^2$ for $\Vert\cdot\Vert$ the Euclidean norm on $\R^d$. Therefore it could only by extended to any $p\geq1$ if we use the norm $\Vert\cdot\Vert_p$ on $\R^d$.

\section{Computation of Wasserstein Distance for a Mixture of Two Dirac Distributions}
\label{sct:integral_expression_mixture}
We derive here integral expressions for $W_2(\L(X),\L(\varphi_\alpha(Y)))$ when $X = \frac{\delta_{-\mu} + \delta_\mu}{2}$ on $\R$ for some $\mu >0$.
In this case, we have
$$
p_Y(y) = \frac{1}{2\sqrt{2\pi \sigma^2}}\op e^{-\frac{(y-\mu)^2}{2\sigma^2}} + e^{-\frac{(y+\mu)^2}{2\sigma^2}}\cp
$$
leading to,
$$
\nabla\log p_Y(y)= \frac{1}{\sigma^2}(-y +\mu\tanh\op\frac{y\mu}{\sigma^2}\cp),
$$
hence
$$
\varphi_\alpha(y) = \alpha \mu \tanh\op\frac{y\mu}{\sigma^2}\cp + (1-\alpha) y
$$
The function $\varphi_\alpha$ is increasing and verifies $\forall y\in\R, \varphi_\alpha(-y)= -\varphi_\alpha(y)$. Moreover, as $X$ is a symmetric random variable ($-X\sim X$), $Y$ is also symmetric, hence $\varphi_\alpha(Y)$ as well. Computing the optimal transport plan is therefore straightforward, as all the mass of $\varphi_\alpha(Y)$ on $\R_-$ should be transported to $-\mu$ and  all the mass of $\varphi_\alpha(Y)$ on $\R_+$ should be transported to $\mu$, each part having the same transport cost, leading to the expression
\begin{align*}
W_2^2(\L(X),\L(\varphi_\alpha(Y))) 
&= 2\E[(\varphi_\alpha(Y) - \mu)^2 \1_{\varphi_\alpha(Y) \geq 0}] \\
&= 2 \cdot\frac{1}{\sqrt{2\pi \sigma^2}}\int_{0}^{\infty}(\varphi_\alpha(y) - \mu)^2 \frac{e^{-\frac{(y-\mu)^2}{2\sigma^2}}+e^{-\frac{(y+\mu)^2}{2\sigma^2}}}{2}dy\\
&= \frac{1}{\sqrt{2\pi \sigma^2}}\int_{0}^{\infty}(\varphi_\alpha(y) - \mu)^2 e^{-\frac{(y-\mu)^2}{2\sigma^2}}dy\\
&\quad+\frac{1}{\sqrt{2\pi \sigma^2}}\int_{0}^{\infty}(\varphi_\alpha(y) - \mu)^2 e^{-\frac{(y+\mu)^2}{2\sigma^2}}dy \\
&= \frac{1}{\sqrt{2\pi \sigma^2}}\int_{0}^{\infty}(\varphi_\alpha(y) - \mu)^2 e^{-\frac{(y-\mu)^2}{2\sigma^2}}dy\\
\text{(change $y \leftarrow-y $)}&\quad+\frac{1}{\sqrt{2\pi \sigma^2}}\int_{-\infty}^{0}(\varphi_\alpha(-y) - \mu)^2 e^{-\frac{(-y+\mu)^2}{2\sigma^2}}dy \\
&= \frac{1}{\sqrt{2\pi \sigma^2}}\int_{-\infty}^{\infty}(\varphi_\alpha(|y|) - \mu)^2 e^{-\frac{(y-\mu)^2}{2\sigma^2}}dy
\end{align*}
We can rearrange this expression to write $\frac{W_2(\L(X),\L(\varphi_\alpha(Y)))}{\mu}$ as a function of $\frac{\sigma}{\mu}$. We first define, for $u \in \R$,
$$
\varphi_\alpha^\sigma(u) = \alpha \tanh\op\frac{u}{\sigma^2}\cp + (1-\alpha)u,
$$
then taking the square root, dividing by $\mu$ and making the change of variable $u = \frac{y}{\mu}$, we get
$$
\frac{W_2(\L(X),\L(\varphi_\alpha(Y)))}{\mu} = \op\frac{1}{\sqrt{2\pi (\sigma/\mu)^2}}\int_{-\infty}^{\infty}(\varphi_\alpha^{\sigma/\mu}(|u|) - 1)^2 e^{-\frac{(y-1)^2}{2(\sigma/\mu)^2}}dy\cp^{1/2}.
$$

\section{Derivation of Deterministic Diffusion-Model Samplers}
\label{sct:derivation_diffusion_models}
First note that the equation for the probability flow ODE given by \citet{karrasElucidatingDesignSpace2022} is (with our notations):
$$
\frac{dx_t}{dt} = \frac{\dot s(t)}{s(t)}x_t - s(t)^2 \dot \sigma(t) \sigma(t) \nabla_x\ob (x,\sigma^2) \mapsto \log p\op \frac{x}{s(t)}; \sigma^2\cp\cb \op x_t; \sigma(t)^2\cp,
$$
which can by rewritten 
$$
\frac{dx_t}{dt} = \frac{\dot s(t)}{s(t)}x_t - s(t) \dot \sigma(t) \sigma(t) \nabla_x\ob (x,\sigma^2) \mapsto \log p\op x; \sigma^2\cp\cb \op \frac{x_t}{s(t)}; \sigma(t)^2\cp.
$$
Writing $\nabla \log p = \nabla_x\ob (x,\sigma^2) \mapsto \log p\op x; \sigma^2\cp\cb$ the score function for the density $p\op x; \sigma^2\cp$ of the normalized variable at noise level $\sigma^2$, we finally get (\ref{eq:ODE_karras}):
\begin{equation*}
\tag{\ref{eq:ODE_karras}}
\frac{dx_t}{dt} = \frac{\dot s(t)}{s(t)}x_t - s(t) \dot \sigma(t) \sigma(t)\nabla\log p\op\frac{x_t}{s(t)}; \sigma(t)^2\cp.
\end{equation*}

\subsection{DDIM Algorithm}
There are different ways to derive the DDIM updates \citep{songDenoisingDiffusionImplicit2021}. Firstly, one can consider a specific process $(X_t)$ with the same marginal as $(x_t)$ defined by (\ref{eq:ODE_karras}). Consider $Z\sim \N(0,I)$ independent from $X$ and define, for $t\geq 0$,
$$
X_t = s(t)(X + \sigma(t)Z)
$$
For $t_k,t_{k+1}\in\R$, we can rewrite
\begin{align*}
X_{t_{k+1}} &= s(t_{k+1})(X + \sigma(t_{k+1})Z)\\
&= s(t_{k+1})\op X + \frac{\sigma(t_{k+1})}{\sigma(t_k)}\op \frac{X_{t_k}}{s(t_k)} - X \cp\cp\\
&= s(t_{k+1})\op 1 - \frac{\sigma(t_{k+1})}{\sigma(t_k)} \cp X + \frac{\sigma(t_{k+1})s(t_{k+1})}{\sigma(t_k)s(t_k)}X_{t_k}
\end{align*}
from which we deduce,
$$
\E[X_{t_{k+1}}|X_{t_k}] = s(t_{k+1})\op 1 - \frac{\sigma(t_{k+1})}{\sigma(t_k)} \cp \E[X|X_{t_k}] + \frac{\sigma(t_{k+1})s(t_{k+1})}{\sigma(t_k)s(t_k)}X_{t_k}.
$$
As $X_{t_k}/s(tk) = X + \sigma(t_k)Z$, Tweedie's formula gives
\begin{align*}
\E[X|X_{t_k}]  
= \E\ob X\Big|\frac{X_{t_k}}{s(t_k)}\cb
&=\frac{X_{t_k}}{s(t_k)} + \sigma(t_k)^2\nabla\log p_{\frac{X_{t_k}}{s(t_k)}}\op \frac{X_{t_k}}{s(t_k)}\cp\\
&=\frac{X_{t_k}}{s(t_k)} + \sigma(t_k)^2\nabla\log p\op \frac{X_{t_k}}{s(t_k)}; \sigma(t_k)^2\cp,
\end{align*}
leading to,
$$
\E[X_{t_{k+1}}|X_{t_k}] = \frac{s(t_{k+1})}{s(t_k)}X_{t_k} + s(t_{k+1})(\sigma(t_{k})^2-\sigma(t_k)\sigma(t_{k+1}))\nabla\log p\op \frac{X_{t_k}}{s(t_k)};\sigma(t_k)^2\cp.
$$
The DDIM update comes from approximating $X_{t_{k+1}}$ by $\E[X_{t_{k+1}}|X_{t_k}]$, leading to 
$$
\hat X_{k+1} = \frac{s(t_{k+1})}{s(t_k)}\hat X_{k} + s(t_{k+1})(\sigma(t_{k})^2-\sigma(t_k)\sigma(t_{k+1}))\nabla\log p\op \frac{\hat X_{k}}{s(t_k)};\sigma(t_k)^2\cp.
$$
Another method to derive the same update is to start from (\ref{eq:ODE_karras}), and take an Euler discretization of (\ref{eq:ODE_karras}) with the approximations $\dot s(t) \approx (s(t_{k+1})-s(t_k))/(t_{k+1}-t_k)$, $\dot \sigma(t) \approx (\sigma(t_{k+1})-\sigma(t_k))/(t_{k+1}-t_k)$, $s(t) \approx s(t_k)$ for the first occurrence of $s(t)$ in (\ref{eq:ODE_karras}) and $s(t) \approx s(t_{k+1})$ for the second, leading to
$$
\frac{\hat X_{k+1} - \hat X_{k}}{t_{k+1}-t_k} =
\frac{1}{s(t_k)}\frac{s(t_{k+1})-s(t_k)}{t_{k+1} - t_k}\hat X_k
- s(t_{k+1})\frac{\sigma(t_{k+1})-\sigma(t_{k})}{t_{k+1}-t_k}\sigma(t_k)\nabla\log p\op\frac{\hat X_k}{s(t_k)}; \sigma(t_k)^2\cp,
$$
that can be rewritten as
$$
\hat X_{k+1} = \frac{s(t_{k+1})}{s(t_k)}\hat X_k + s(t_{k+1})(\sigma(t_k)^2- \sigma(t_k)\sigma(t_{k+1})) \nabla\log p\op\frac{\hat X_k}{s(t_k)}; \sigma(t_k)^2\cp.
$$
The procedure proposed by \citet{songDenoisingDiffusionImplicit2021} is recovered by taking $s(t)= \sqrt{\alpha_t}$ and $\sigma(t) = \frac{\sqrt{1-\alpha_t}}{\sqrt{\alpha_t}}$.

\subsection{Euler Discretization}
The Euler discretization of (\ref{eq:ODE_karras}) is
$$
\frac{\hat X_{k+1} - \hat X_{k}}{t_{k+1}-t_k} =
\frac{\dot s(t_k)}{s(t_k)}\hat X_k
- s(t_{k})\dot\sigma(t_k)\sigma(t_k)\nabla\log p\op\frac{\hat X_k}{s(t_k)}; \sigma(t_k)^2\cp,
$$
which can by rewritten 
$$
\hat X_{k+1} = \op 1 + \frac{\dot s(t_{k})(t_{k+1}-t_k)}{s(t_k)}\cp\hat X_k - s(t_{k}) (t_{k+1}-t_k) \dot \sigma(t_k) \sigma(t_k) \nabla\log p\op\frac{\hat X_k}{s(t_k)}; \sigma(t_k)^2\cp.
$$

\vskip 0.2in
\bibliography{biblio.bib}

\end{document}